%% file: main.tex
\documentclass[11pt]{article}
\usepackage{iclr2024_conference, times}
\input{packages}
\input{symbols}

\setcitestyle{authoryear,open={(},close={)}}

\newcommand{\ignore}[1]{}

\newcommand{\atkname}{Chameleon}

\begin{document}

\title{Chameleon: Increasing Label-Only Membership Leakage with Adaptive Poisoning}

\author{Harsh Chaudhari, Giorgio Severi, Alina Oprea, Jonathan Ullman 
% \thanks{ Use footnote for providing further information
% about author (webpage, alternative address)---\emph{not} for acknowledging
% funding agencies.  Funding acknowledgements go at the end of the paper.} 
\\
Khoury College of Computer Science\\
Northeastern University\\
% Boston, MA 02115, USA \\
\texttt{\{chaudhari.ha, severi.g, a.oprea, j.ullman\}@northeastern.edu} \\
}

\iclrfinalcopy % Uncomment for camera-ready version, but NOT for submission.

\maketitle

\begin{abstract}
\noindent
The integration of machine learning (ML) in numerous critical applications introduces a range of privacy concerns for individuals who provide their datasets for model training. One such privacy risk is Membership Inference (MI), in which an attacker seeks to determine whether a particular data sample was included in the training dataset of a model. Current state-of-the-art MI attacks capitalize on access to the model’s predicted confidence scores to successfully perform membership inference, and employ data poisoning to further enhance their effectiveness. 
In this work, we  focus on the less explored and more realistic \emph{label-only} setting, where the model provides only the predicted label on a queried sample. We show that existing label-only MI attacks are ineffective at inferring membership in the low False Positive Rate (FPR) regime. To address this challenge,  we propose a new attack \atkname\ that leverages a novel adaptive data poisoning strategy and an efficient query selection method to achieve significantly more accurate membership inference than existing label-only attacks, especially at low FPRs.

 \end{abstract}

\input{Introduction}
\input{Background}
\input{ThreatModel}

\input{Overview}
\input{multipoint_alg}

\input{experiments}

\input{ablations}

\input{DP}

\section{Discussion and Conclusion}

% We propose a new attack that successfully amplifies Membership Inference leakage in the Label-Only setting. Our attack leverages a novel adaptive poisoning and querying strategy, surpassing the effectiveness of prior label-only attacks.  Furthermore, we investigate the viability of Differential Privacy as a defense against our attack, considering its impact on model utility. Finally, we offer a theoretical analysis providing insights on the impact of data poisoning on MI leakage. 

In this work we propose a new attack that successfully amplifies Membership Inference leakage in the Label-Only setting. Our attack leverages a novel adaptive poisoning and querying strategy, surpassing the effectiveness of prior label-only attacks.  Furthermore, we investigate the viability of Differential Privacy as a defense against our attack, considering its impact on model utility. Finally, we offer a theoretical analysis providing insights on the impact of data poisoning on MI leakage. 

%While \atkname~demonstrates impressive performance in our experiments, it is essential to acknowledge certain limitations. First, our attack relies on data poisoning to succeed in the Membership Inference test.  While poisoning is a core component of our approach, and allows us to enhance the privacy leakage, it also imposes burdens on the adversary, and comes at the expense of some, albeit limited, model utility.
\medskip
We demonstrated that \atkname\ achieves impressive performance in our experiments, mainly due to the adaptive poisoning strategy we design. While poisoning is a core component of our approach, and allows us to enhance the privacy leakage, it also imposes additional burden on the adversary to mount the poisoning attack. 
One important remaining open problem in label-only MI attacks is how to operate effectively in the low False Positive Rate (FPR) scenario without the assistance of poisoning. 
Additionally, our poisoning strategy requires training shadow models. 
Though our approach generally involves a low number of shadow models,  any training operation is inherently expensive and adds computational complexity. 
An interesting direction for future work is the design of poisoning strategies for label-only membership inference that do not require shadow model training.

% \newpage
\section*{Acknowledgements}
\noindent
We thank Sushant Agarwal and John Abascal for helpful discussions.  Alina Oprea was supported by NSF awards CNS-2120603 and CNS-2247484. Jonathan Ullman was supported by NSF awards CNS-2120603, CNS-2232692, and CNS-2247484.

% \newpage 

\ignore{
\section*{Ethics Statement}
Our work introduces a new membership inference methodology, that could in principle be used to exacerbate the privacy risks to individuals' private data when included in machine learning training sets.
The main reason motivating us to investigate this type of attacks is to shed light on the extent of the risks, and ensure practitioners are aware of them.
In comparison with previous works investigating label-only poisoning-enhanced membership inference attacks, we show that the adversary can manage to reliably compromise the privacy of a significant fraction of the training data.
We argue that this type of privacy risk should be seriously taken into account and encourage practitioners to adopt private training procedures, such as differentially private training, to minimize it.
}

% \newpage
\bibliography{iclr2024_conference}

\appendix
% \section{Appendix}
\input{apndx_ablations}

\input{apndx_threatmodel}
\input{apndx_analysis}

\end{document}

%% file: packages.tex
\usepackage{adjustbox}
\usepackage{amsfonts}
\usepackage{amsmath}
\usepackage{amsthm}
\usepackage{algorithm}
\usepackage{booktabs}
\usepackage{breqn}
\usepackage{bbm}
\usepackage{caption}
\usepackage{cite}
\usepackage{comment}
\usepackage{derivative}
\usepackage{enumitem}
\usepackage[T1]{fontenc}
\usepackage{filecontents}
\usepackage{graphicx}
\usepackage{hyperref}
\usepackage{longtable}
\usepackage{lipsum}
\usepackage{makecell}
\usepackage{microtype}
\usepackage{multirow}
\usepackage{multicol}
\usepackage{pgfplots}
\pgfplotsset{compat=1.16}
\usepackage{pgfplotstable}
\usepackage{pifont}
\usepackage{stmaryrd}
\usepackage{subcaption}
\usepackage{scalerel}
\usepackage{tabu}
\usepackage{textcomp}
\usepackage{colortbl}
\usepackage{tikz}
\usepackage[framemethod=tikz]{mdframed} 
\usetikzlibrary{shapes.geometric,arrows}
\usetikzlibrary{calc}
\usepackage{thmtools}
\usepackage{thm-restate}
\usepackage{url}
\usepackage{wrapfig}
\usepackage{xcolor}
\usepackage{xspace}
\usepackage{algpseudocode}
\usepackage{cleveref}

%% file: symbols.tex
\newcommand{\dist}{\mathcal{D}}

\newcommand{\psnthresh}{\mathsf{{t_p}}}
\newcommand{\nbrthresh}{\mathsf{{t_{nb}}}}
\newcommand{\nbrset}[1]{\mathsf{S}^{#1}_{\mathsf{nb}}}
\newcommand{\advdata}{\mathsf{{D_{adv}}}}
\newcommand{\psndata}{\mathsf{{D_{p}}}}
\newcommand{\cham}{\text{Chameleon }}
\newcommand{\traindata}{\mathsf{{D_{tr}}}}
\newcommand{\tgtmodel}{{{M_{t}}}}

\newtheorem{theorem}{Theorem}[section]

\newcommand{\abs}[1]{| #1 |}

\newcommand{\Adv}{\ensuremath{\mathcal{A}}\xspace}

\newcommand{\challenger}{\ensuremath{\mathcal{C}}\xspace}

\makeatletter
\newsavebox{\@brx}
\newcommand{\llangle}[1][]{\savebox{\@brx}{\(\m@th{#1\langle}\)}%
	\mathopen{\copy\@brx\kern-0.6\wd\@brx\usebox{\@brx}}}
\newcommand{\rrangle}[1][]{\savebox{\@brx}{\(\m@th{#1\rangle}\)}%
	\mathclose{\copy\@brx\kern-0.6\wd\@brx\usebox{\@brx}}}
\makeatother

\newcommand{\Sh}{\ensuremath{\mathsf{sh}}}

\newcommand{{\piaSh}}[1]{\ensuremath{\Pi^{#1}_{\Sh}}}

\definecolor{lightmintbg}{rgb}{.53,.80,.92}

\tikzstyle{maldo} = [rectangle, minimum width=2.5cm, minimum height=0.2cm, text centered, draw=black, fill=orange!75]
\tikzstyle{hdo} = [rectangle,,minimum width=2.5cm, minimum height=0.25cm,text centered, draw=black, fill = lightmintbg!60]
\tikzstyle{malserver} = [rectangle,minimum width=0.7cm, minimum height=0.3cm, text centered, draw=black, fill=red!60]

\tikzstyle{hserver} = [rectangle,minimum width=0.7cm, minimum height=0.3cm, text centered, draw=black, fill=lightmintbg!60]
\tikzstyle{soc} = [rectangle, rounded corners,dashed,minimum width=3.6cm, minimum height=2.6cm, draw=black]

\tikzstyle{nota} = [rectangle, minimum width=8.5cm, minimum height= 0.4cm, font = \small, draw=black]

\tikzstyle{mdnota} = [rectangle, minimum width=0.4cm, minimum height=0.2cm, text centered,font = \small, draw=black, fill=orange!75]

\tikzstyle{msnota} = [rectangle,minimum width=0.4cm, minimum height=0.2cm, text centered, font = \small, draw=black, fill=red!60]

\tikzstyle{hnota} = [rectangle, minimum width=0.4cm, minimum height=0.2cm, text centered,font = \small, draw=black, fill=lightmintbg!60]
\tikzstyle{sarrow} = [ultra thin, <->,latex-latex]
\tikzstyle{darrow} = [thin,->,>=stealth]

\newlength{\maxlen}

\setlist[description]{style=unboxed,leftmargin=0cm}

\newcounter{itemcount}

\newcommand{\figlab}[1]{\label{fig:#1}}

\newenvironment{boxfig*}[2]{% {#1}{#2} = {Caption}{label}
	\begin{figure*}[h!]		
		\fontsize{5}{5}\selectfont
		\newcommand{\FigCaption}{#1}
		\newcommand{\FigLabel}{#2}
		\vspace{-.05cm}
		\begin{center}
			\begin{small}			 
				\begin{adjustbox}{max width=\textwidth}
					\begin{tabular}{@{}|@{~~}l@{~~}|@{}}
						\hline
						\rule[-1ex]{0pt}{1ex}\begin{minipage}[b]{.95\linewidth}
							\vspace{1ex}	
						}{%
						\end{minipage}\\
						\hline
					\end{tabular}	
				\end{adjustbox}		
			\end{small}
			\vspace{-0.1cm}
			\caption{\FigCaption}
			\figlab{\FigLabel}
		\end{center}
		\vspace{-.38cm}
	\end{figure*}
}

\newenvironment{myboxfig*}[2]{% {#1}{#2} = {Caption}{label}
	\begin{figure*}[!htb]		
		\fontsize{5}{5}\selectfont
		\newcommand{\FigCaption}{#1}
		\newcommand{\FigLabel}{#2}
		\vspace{-.10cm}
		\begin{center}
			\caption{\FigCaption}
			\begin{small}			 
				\begin{adjustbox}{max width=\textwidth}
					\begin{tabular}{@{}|@{~~}l@{~~}|@{}}
						\hline
						\rule[-1ex]{0pt}{1ex}\begin{minipage}[b]{.95\linewidth}
							\vspace{1ex}	
						}{%
						\end{minipage}\\
						\hline
					\end{tabular}	
				\end{adjustbox}		
			\end{small}
			\vspace{-0.25cm}
			\figlab{\FigLabel}
		\end{center}
		\vspace{-.38cm}
	\end{figure*}
}

\RequirePackage{mdframed}

\newenvironment{titlebox}[5]
{\mdfsetup{
		style=#2,
		innertopmargin=1.1\baselineskip,
		skipabove={\dimexpr0.7\baselineskip+\topskip\relax},
		skipbelow={1.5em},needspace=3\baselineskip,
		singleextra={\node[#3,right=10pt,overlay] at (P-|O){~{\sffamily\bfseries #1 }};},%
		firstextra={\node[#3,right=10pt,overlay] at (P-|O) {~{\sffamily\bfseries #1 }};},
		frametitleaboveskip=9em,
		innerrightmargin=5pt
	}
	\newcommand{\TitleCaption}{#4}
	\newcommand{\TitleLabel}{#5}
	\begin{mdframed}[font=\small]
		\setlist[itemize]{leftmargin=13pt}\setlist[enumerate]{leftmargin=13pt}\raggedright% 
	}
	{\end{mdframed}
	\vspace{-1.5em}
	{\captionof{figure}{\normalfont \TitleCaption}\label{\TitleLabel}}
	\medskip
}

\tikzstyle{normal} = [thick, fill=white, text=black, draw, rounded corners, rectangle, minimum height=.7cm, inner sep=3pt]
\tikzstyle{gray} = [thick, fill=gray!90, text=white, rounded corners, rectangle, minimum height=.7cm, inner sep=3pt]

\mdfdefinestyle{commonbox}{%
	align=center, middlelinewidth=1.1pt,userdefinedwidth=\linewidth
	innerrightmargin=0pt,innerleftmargin=3pt,innertopmargin=5pt,
	splittopskip=7pt,splitbottomskip=5pt
}
\mdfdefinestyle{roundbox}{style=commonbox,roundcorner=5pt,userdefinedwidth=\linewidth}

\newenvironment{systembox}[3]
{ \begin{titlebox}{\normalfont #1}{roundbox}{normal}{#2}{#3}}
	{\end{titlebox}}

\newenvironment{gsystembox}[3]
{\vspace{\baselineskip}\begin{titlebox}{Global Functionality \normalfont #1}{roundbox}{normal}{#2}{#3}}
	{\end{titlebox}}

\newenvironment{protocolbox}[3]
{\begin{titlebox}{Protocol \normalfont #1}{commonbox}{normal}{#2}{#3}}
	{\end{titlebox}}

\newenvironment{algobox}[3]
{\begin{titlebox}{Algorithm \normalfont #1}{commonbox}{normal}{#2}{#3}}
	{\end{titlebox}}

\newenvironment{reductionbox}[3]
{\begin{titlebox}{Reduction \normalfont #1}{commonbox}{normal}{#2}{#3}}
	{\end{titlebox}}

\newenvironment{gamebox}[3]
{\begin{titlebox}{Game \normalfont #1}{commonbox}{gray}{#2}{#3}}
	{\end{titlebox}}

\newenvironment{simulatorbox}[3]
{\begin{titlebox}{Simulator \normalfont #1}{commonbox}{normal}{#2}{#3}}
	{\end{titlebox}}
\newenvironment{systembox*}[3]
{\begin{strip}
		\vspace{\baselineskip}\begin{titlebox}{Functionality \normalfont #1}{roundbox}{normal}{#2}{#3}}
		{\end{titlebox}
\end{strip}}

\newenvironment{gsystembox*}[3]
{\begin{strip}
		\vspace{\baselineskip}\begin{titlebox}{Global Functionality \normalfont #1}{roundbox}{normal}{#2}{#3}}
		{\end{titlebox}
\end{strip}}

\newenvironment{protocolbox*}[3]
{\begin{strip}
		\begin{titlebox}{Protocol \normalfont #1}{commonbox}{normal}{#2}{#3}}
		{\end{titlebox}
\end{strip}}

\newenvironment{algobox*}[3]
{\begin{strip}
		\begin{titlebox}{Algorithm \normalfont #1}{commonbox}{normal}{#2}{#3}}
		{\end{titlebox}
\end{strip}}

\newenvironment{reductionbox*}[3]
{\begin{strip}
		\begin{titlebox}{Reduction \normalfont #1}{commonbox}{normal}{#2}{#3}}
		{\end{titlebox}
\end{strip}}

\newenvironment{gamebox*}[3]
{\begin{strip}
		\begin{titlebox}{Game \normalfont #1}{commonbox}{gray}{#2}{#3}}
		{\end{titlebox}
\end{strip}}

\newenvironment{simulatorbox*}[3]
{\begin{strip}
		\begin{titlebox}{Simulator \normalfont #1}{commonbox}{normal}{#2}{#3}}
		{\end{titlebox}
\end{strip}}

\mdfdefinestyle{mycommonbox}{%
	align=center, middlelinewidth=1.1pt,userdefinedwidth=\linewidth
	innerrightmargin=0pt,innerleftmargin=2pt,innertopmargin=3pt,
	splittopskip=7pt,splitbottomskip=3pt
}
\mdfdefinestyle{myroundbox}{style=mycommonbox,roundcorner=5pt,userdefinedwidth=\linewidth}

\newenvironment{titlebox*}[5]
{\mdfsetup{
		style=#2,
		innertopmargin=0.3\baselineskip,
		skipabove={1.2em},
		skipbelow={1em},needspace=3\baselineskip,
		frametitleaboveskip=5em,
		innerrightmargin=5pt
	}
	\newcommand{\TitleCaption}{#4}
	\newcommand{\TitleLabel}{#5}
	\begin{mdframed}[font=\small]
		\setlist[itemize]{leftmargin=13pt}\setlist[enumerate]{leftmargin=13pt}\raggedright% 
	}
	{\end{mdframed}
	\vspace{-1.5em}
	{\captionof{figure}{\normalfont \TitleCaption}\label{\TitleLabel}}
	\medskip
}

\newenvironment{mysystembox*}[3]
{\begin{strip}
		\vspace{\baselineskip}\begin{titlebox*}{Functionality \normalfont #1}{myroundbox}{normal}{#2}{#3}}
		{\end{titlebox*}
\end{strip}}

\newenvironment{mygsystembox*}[3]
{\begin{strip}
		\vspace{\baselineskip}\begin{titlebox*}{Global Functionality \normalfont #1}{myroundbox}{normal}{#2}{#3}}
		{\end{titlebox*}
\end{strip}}

\newenvironment{myprotocolbox*}[3]
{\begin{strip}
		\begin{titlebox*}{Protocol \normalfont #1}{mycommonbox}{normal}{#2}{#3}}
		{\end{titlebox*}
\end{strip}}

\newenvironment{myalgobox*}[3]
{\begin{strip}
		\begin{titlebox*}{Algorithm \normalfont #1}{mycommonbox}{normal}{#2}{#3}}
		{\end{titlebox*}
\end{strip}}

\newenvironment{myreductionbox*}[3]
{\begin{strip}
		\begin{titlebox*}{Reduction \normalfont #1}{mycommonbox}{normal}{#2}{#3}}
		{\end{titlebox*}
\end{strip}}

\newenvironment{mygamebox*}[3]
{\begin{strip}
		\begin{titlebox*}{Game \normalfont #1}{mycommonbox}{gray}{#2}{#3}}
		{\end{titlebox*}
\end{strip}}

\newenvironment{mysimulatorbox*}[3]
{\begin{strip}
		\begin{titlebox*}{Simulator \normalfont #1}{mycommonbox}{normal}{#2}{#3}}
		{\end{titlebox*}
\end{strip}}

\newcommand{\algoHead}[1]{\vspace{0.2em} \underline{\textbf{#1}} \vspace{0.3em}}

\makeatletter
\algnewcommand{\ExtendedState}[1]{\State
	\parbox[t]{\dimexpr\linewidth-\ALG@thistlm}{\hangindent=\algorithmicindent\strut\hangafter=3#1\strut}}
\makeatother

\algnewcommand\algorithmicinput{\textbf{Input:}}
\algnewcommand\Input{\item[\algorithmicinput]}

\algrenewcommand{\algorithmiccomment}[1]{{\color{gray}// #1}}

\newcommand{\xmath}[1]{\ensuremath{#1}\xspace}

\newcommand{\Func}[1][\relax]{\xmath{\mathcal{F}_{\textsc{#1}}}}

%% file: Introduction.tex
\section{Introduction}
% \harsh{Rephrase the first paragraph: Add one sentence to mention other privacy attacks; 1-2 sentences to explain why MI is a privacy violation. 
% In addition, somewhere in the intro you can make the point that many cloud APIs only provide labels (need to do a search to confirm and need to cite them), so the threat model is realistic.}

% Combining ML in a variety of software applications, including those dealing with potentially sensitive user data, individuals contributing their data to populate training datasets are exposed to a variety of potential privacy risks.... 

The use of machine learning for training on confidential or sensitive data, such as medical records \citep{stanfill2010systematic}, financial documents \citep{finance11}, and conversations \citep{Extract21},  introduces a range of privacy violations. By interacting with a trained ML model, an attacker might reconstruct data from the training set~\citep{NIPS22, Oakland22}, perform membership inference~\citep{shokri2017membership, yeom2018privacy, LiRA}, learn sensitive attributes~\citep{atri15, atr22} or global properties~\citep{Ganju18,DistributionInference22} from training data. Membership inference (MI) attacks \citep{shokri2017membership}, originally introduced under the name of tracing attacks~\citep{Homer+08}, enable an attacker to determine whether or not a data sample was included in the training set of an ML model. While these attacks are less severe than training data reconstruction, they might still constitute a serious privacy violation.  
Consider a mental health clinic that uses an ML model to predict patient treatment responses based on medical histories. An attacker with accesses to a certain individual's medical history can learn if the individual has a mental health condition, by performing a successful MI attack. 

\smallskip
We can categorize  MI attacks into two groups: Confidence-based attacks in which the attacker gets access to the target ML model's predicted confidences, and label-only attacks, in which the attacker only obtains the predicted label on queried samples. Recent literature has primarily focused on confidence-based attacks~\citet{LiRA, bertran2023scalable}  that maximize the attacker's success at low False-Positive Rates (FPRs). 
Additionally,  \citet{TruthSerum} and \citet{MIPoison22}  showed that introducing data poisoning during training significantly improves the MI performance at low FPRs in the confidence-based scenario.

\smallskip
Nevertheless, in many real-world scenarios, organizations that train ML models provide only hard labels to customer queries. For example, financial institutions might solely indicate whether a customer has been granted a home loan or credit card approval. In such situations, launching an MI attack gets considerably more challenging as the attacker looses access to prediction confidences and cannot leverage state-of-the-art attacks such as \citet{LiRA}, \citet{wen2023canary}, \citet{bertran2023scalable}.  
Furthermore, it remains unclear whether existing label-only MI attacks, such as  \citet{choquette-choo21a} and \citet{LO2}, are effective in the low FPR regime and whether data poisoning techniques can be used to amplify the membership leakage in this specific realistic scenario.  

% Earlier works in this literature \citep{shokri2017membership, yeom2018privacy, hayes18, Milad18, sablayrolles2019whitebox, salem2019mlleaks, leino2020stolen, song2021systematic} primarily aimed to improve average-case success metrics, such as Membership Inference accuracy, which provide an overall measure of the attack's success across the entire dataset.
% However, recent studies \citep{LiRA, ye2022enhanced} argue  that privacy is not an average case metric and redirected their focus toward designing algorithms that maximize the attacker's success in the low False-Positive (FPR) regime. Additionally, works by \citet{TruthSerum} and \citet{MIPoison22} have demonstrated that introducing data poisoning during the model training phase can significantly enhance the performance of a MI attack at low FPRs. For instance \citet{TruthSerum} achieves approximately $8\times$ improvement over \citet{LiRA} at $0.1\%$ FPR.

% These strategies, however, were designed for scenarios where the attacker has access to the confidence values given by the target model. 
% In the case of Label-Only Membership Inference, where the attacker only obtains the predicted labels from the model, state-of-the-art attacks \citep{choquette-choo21a,LO2} are primarily focused on maximizing average-case success metrics. However, it remains unknown whether these attacks are effective in the low FPR regime and if data poisoning can be used to amplify the membership leakage in the label-only setting. 

%\gio{We should mention that this is a shadow-model based attack}

\smallskip
In this paper, we first show that existing label-only MI attacks  \citep{yeom2018privacy, choquette-choo21a, LO2} struggle to achieve high True Positive Rate (TPR) in the low FPR regime. We then demonstrate that integrating state-of-the-art data poisoning technique \citep{TruthSerum} into these label-only MI attacks further degrades their performance, resulting in even lower TPR values at the same FPR. We  investigate the source of this failure and  propose a new label-only MI attack  \atkname\ that leverages a novel \emph{adaptive} poisoning strategy to enhance membership inference leakage in the label-only setting. Our attack also uses an \emph{efficient} querying strategy, which requires only $64$ queries to the target model to succeed in the distinguishing test, unlike prior works \citep{choquette-choo21a, LO2} that use  on the order of a few thousand queries. 
Extensive experimentation across multiple datasets shows that our \atkname~attack consistently outperforms previous label-only MI attacks, with improvements in TPR at 1\% FPR ranging up to $17.5\times$. Finally, we also provide a theoretical analysis that sheds light on how data poisoning amplifies membership leakage in label-only scenario. To the best of our knowledge, this work represents the first analysis on understanding the impact of poisoning on MI attacks.

% \harsh{Also add a line on having a theoretical analysis of the effect of poisoning for MI leakage.}

%% file: Background.tex
\section{Background and Threat Model}
\label{sec:background}

We provide background on membership inference, describe our label-only threat model with poisoning, and analyze existing approaches to motivate our new attack. 
% \oparagraph{Privacy attacks}

% \paragraph{Poisoning attacks}

% \paragraph{Amplifying Privacy attacks with Poisoning}

% \paragraph{Membership Inference}
% Privacy violations in Machine Learning can assume a variety of different shapes such as data extraction, attribute inference, and membership inference.
% \emph{Membership Inference} attacks \citep{shokri2017membership} are amongst the most notable and well-known types of privacy violations in Machine Learning systems.
% The objective of MI is to infer whether or not a data point $x$ was included in the training set $\traindata$ of a given trained model $M$.

\paragraph{Related Work.}
\emph{Membership Inference} attacks can be characterized into different types based on the level of adversarial knowledge required for the attack. 
Full-knowledge (or white-box) attacks \citep{nasr2018comprehensive,leino2020stolen} assume the adversary has access to the internal weights of the model, and therefore the activation values of each layer. 
In black-box settings the adversary can only query the ML model, for instance through an API, which may return either confidence scores or hard labels. 
The confidence setting has been  studied most, with works like \citet{shokri2017membership,LiRA,ye2022enhanced} training multiple shadow models —local surrogate models— and modeling the loss (or logit) distributions for members and non-members. 

The \emph{label-only} MI setting,  investigated by \citet{yeom2018privacy, choquette-choo21a,LO2}, models a more realistic threat models that returns only the predicted label on a queried sample. Designing MI attacks under this threat model is more challenging, as the attack cannot rely  on separating the model's confidence on members and non-members. 
%the  it prevents the approaches based on loss modeling, while simultaneously modeling more closely the reality of many machine learning models exposed through public-facing APIs.
Existing label-only MI attacks are based on analyzing the effects of perturbations on the original point on the model's decision.
With our work we aim to improve the understanding of MI in the label-only setting, especially in light of recent trends in MI literature.

Current MI research, in fact, is shifting the attention towards attacks that achieve high True Positive Rates (TPR) in low False Positive Rates (FPR) regimes \citep{LiRA, ye2022enhanced, LossTraj22, wen2023canary, bertran2023scalable}. 
These recent papers argues that if an attack can manage to reliably breach the privacy of even a small number of, potentially vulnerable, users, it is still extremely relevant, despite resulting in potentially lower average-case success rates.
%Moreover, a high aggregate success rate may mask inconsistent results on specific points.
% \paragraph{Poisoning and training privacy}
A second influential research thread exposed the effect that training data poisoning  has on amplifying privacy risks.
This threat model is particularly relevant when the training data is crowd-sourced, or obtained through automated crawling (common for large datasets), as well as  in collaborative learning settings. 
%All scenarios where it is easy for an adversary to introduce poisoned data. 
\citet{TruthSerum} and \citet{MIPoison22} showed that data poisoning amplifies MI privacy leakage and increases the TPR values at low FPRs. Both LiRA \citep{LiRA} and Truth Serum \citep{TruthSerum} use a large number of shadow models (typically 128) to learn the distribution of model confidences, but these methods do not directly apply to label-only membership inference, a much more challenging setting.
%\citet{TruthSerum} and \citet{MIPoison22} showed that using poisoning attacks an adversary can simplify the task of distinguishing between the case where the target point was included in the training set from the one where it was not included, obtaining large improvements over prior attacks without poisoning \citep{LiRA}.

A related line of research by \citet{PIP} and \citet{SNAP}  showcased how data poisoning could be utilized to amplify the leakage of statistical information about the overall properties of the training set,  called property inference attacks.
% We design a poisoning methodology that enhances the privacy leakage of the victim model, and is especially effective at low FPR rates.

%% file: ThreatModel.tex
%\section{Problem Statement}
%\label{sec:statement}

%\gio{This "Motivation" paragraph is quite redundant. I would consider removing it, which would also save us a bit of space.}

%\noindent {\bf Motivation.}
%In the field of membership inference attacks, there has been a recent shift towards maximizing the True Positive Rate (TPR) at low (e.g. $\leq 1\%$) False Positive Rate (FPR) regimes. Current state-of-the-art attacks \citep{LiRA,wen2023canary}, that utilize model confidences, have shown remarkable success in achieving high TPR values in these regimes.  The effectiveness of confidence-based attacks has been further enhanced by incorporating data poisoning techniques to amplify membership inference leakage \citep{TruthSerum,MIPoison22}. 

%However, existing label-only MI attacks \citep{yeom2018privacy, choquette-choo21a,LO2}, which rely solely on the predicted labels without access to confidence scores, are designed to maximize the average-case accuracy metric  without explicitly considering the performance of TPR in low FPR regimes. This poses an intriguing question: Do existing label-only attacks perform well in low FPR regimes? If not, can we leverage data poisoning techniques to enhance their performance?

\paragraph{Threat Model.} 
We follow the threat model of \citet{TruthSerum} used for membership inference with data poisoning, with adjustments to account for the more realistic label-only  setting. The attacker has black-box query access ---the ability to send samples and obtain the corresponding outputs--- to a trained machine learning model $\tgtmodel$, also called target model, that returns only the predicted label on an input query. The attacker's objective is to determine whether a particular target sample was part of $\tgtmodel$'s training set or not. 
Similarly to \citet{TruthSerum} the attacker $\Adv$  has the capability to inject additional poisoned data $\psndata$ into the training data $\traindata$ sampled from a data distribution $\dist$. The attacker can only inject $\psndata$ once before the training process begins, and the adversary does not participate further in the training process after injecting the poisoned samples. After training completes, the adversary can only interact with the final trained model to obtain predicted labels on selected queried samples. Following the MI literature, $\Adv$ can also train local shadow models with data from the same distribution as the target model's training set. Shadow models training sets may or may not include the challenge points. We will call IN models those trained with the challenge point included in the training set, and OUT models those trained without the challenge point.

%Consider a challenger $\challenger$ that samples training data $\traindata \sim \dist$ from an underlying data distribution $\dist$. Also consider an attacker $\Adv$ who has the capability to inject additional poisoned data $\psndata$ into the training data $\traindata$. The objective of the attacker is to enhance its ability to infer if a specific point $(x,y)$ is present in the training data by interacting with a model trained by challenger  $\challenger$ on data $\traindata \cup \psndata$.  The attacker can only inject $\psndata$ once before the training process begins, and after training, it can only interact with the final trained model to obtain predicted labels. Note that, both the challenger and the attacker have access to the underlying data distribution $\dist$, and know the challenge point $(x,y)$ and training algorithm $\mathcal{T}$, similar to prior works \citep{LiRA,TruthSerum,MIPoison22,wen2023canary}. We defer the formal steps of the privacy game to Appendix \ref{apdx:threatmodel}.

\paragraph{Analyzing Existing Approaches.} 
Existing label-only MI attacks \citep{choquette-choo21a,LO2} propose a decision boundary technique that exploits the existence of adversarial examples to create their distinguishing test. These approaches typically require a large number of queries to the target model to estimate a sample's distance to the model decision boundary. However, these attacks  achieve low TPR (e.g., 1.1\%) at 1\% FPR , when tested on the CIFAR-10 dataset. In contrast, the LiRA confidence-based attack by \citet{LiRA} achieves a TPR of 16.2\% at 1\% FPR on the same dataset. Truth Serum \citep{TruthSerum} evaluates LiRA with a data poisoning strategy based on label flipping, which significantly increases TPR to 91.4\% at 1\% FPR   once 8 poisoned samples are inserted per challenge point. 

A natural first strategy for label-only MI with poisoning is to incorporate the Truth Serum  data poisoning method to the existing label-only MI attack \citep{choquette-choo21a} and investigate if the TPR at low FPR can be improved. The Truth Serum poisoning strategy is simply label flipping, where poisoned samples have identical features to the challenge point, but a different label. Surprisingly, the results show a negative outcome, with the TPR decreasing to 0\% at 1\% FPR after poisoning.  This setback compels us to reconsider  the role of data poisoning in improving privacy attacks within the label-only MI threat model.  We question whether data poisoning can indeed improve label-only MI, and if so, why did our attempt to combine the two approaches fail.  In the following section, we  provide comprehensive answers to these questions and present a novel poisoning strategy that significantly improves the attack success in the label-only MI setting.

% The underlying assumption behind these attacks is that training samples generally require a higher level of perturbation in order to be misclassified compared to test samples.

%1.1\% TPR at 1.02% FPR for Label-only attacks. Truth Serum with 8 poisons: 91.4% at 1% FPR. Combunation is 0% TPR at 1% FPR.

% Pointers: 

% 1) Show SOTA label-only attacks (without poisoning) fail in the low FPR regime. 

% 2) Naively combining  best confidence attack with poisoning (Truth Serum) and best label-only attack also do not work directly.

% 3) Connect to next section saying  can poisoning really help  amplifying membership leakage in our threat model? If so, why did the previous approach fail? We provide answers for this in the next section. 

%% file: Overview.tex
\section{\cham Attack}

\begin{figure}[t] %this figure will be at the right
{    \centering
    \includegraphics[width=\textwidth]{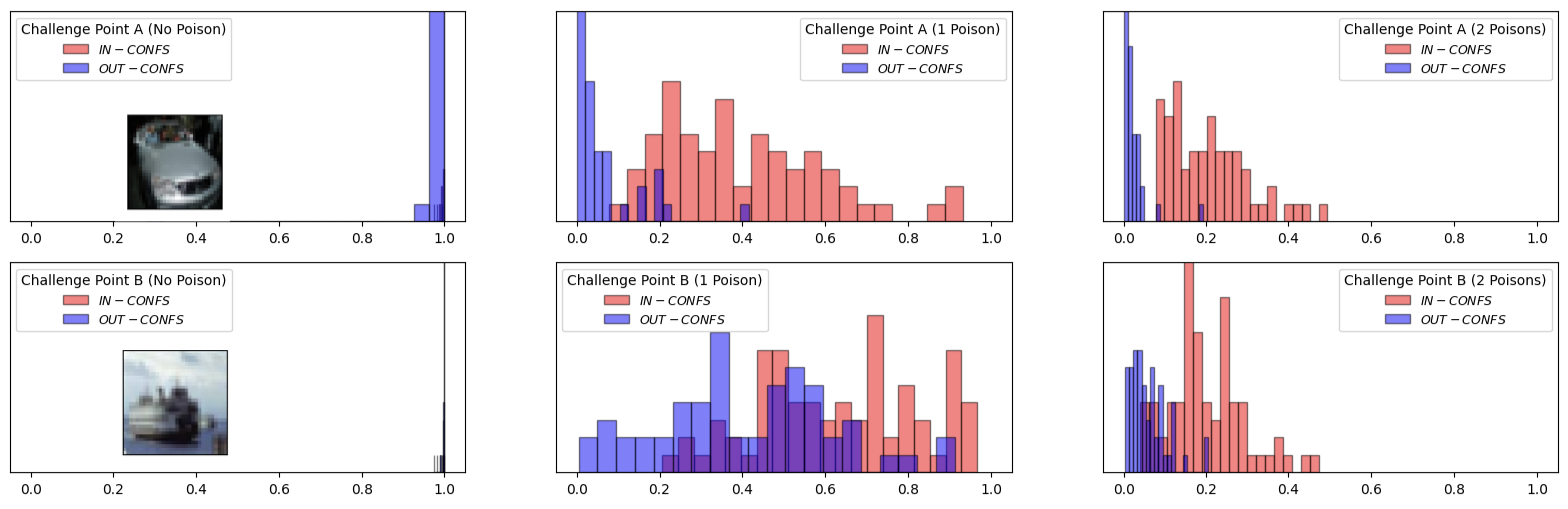}
    \caption{{Distribution of confidence scores for two challenge points (a Car and a Ship), highlighting the impact} of poisoning on two different points of the CIFAR-10 dataset. Each row denotes the shift in model confidences (wrt. true label) for IN and OUT models  with introduction of poisoned samples.}
    \vspace{-2mm}
    \label{fig:AP-Conf}
}    
\end{figure}

We first provide some key insights for our attack, then describe the detailed attack procedure, and include some analysis on leakage under MI. 
\subsection{Attack Intuition} 
\label{sec:Intuit}
% Considering the threat model described in section \ref{sec:threatmodel}, our main objective is to significantly improve the true positive rate at low (e.g.  $\leq 1\%$) false positive rates, while considerably reducing the number of queries needed to interact with the target model to succeed the membership inference test.

Given the threat model, our main objective is to improve the TPR in the \emph{low} FPR regime for label-only MI,  while reducing the number of queries to the target model $\tgtmodel$. To achieve this two-fold objective, we start by addressing the fundamental question of determining an effective poisoning strategy. This involves striking the right balance such that the IN models, those trained with the target point included in the training set, classify the point correctly, and the OUT models, trained without the target point, misclassify it. Without any poisoning, it is likely that both IN and OUT models will classify the point correctly (up to some small training and generalization error). On the other hand, if we insert too many poisoned samples with an incorrect label, then both IN and OUT models will mis-classify the point to the incorrect label. As the attacker only gets access to the labels of the queried samples from the target model $M_t$, over-poisoning would make it implausible to distinguish whether the model is an  IN or OUT model. 

%where, if the model was trained on the  dataset without the challenge point, the model would misclassify it, whereas if the challenge point was included in the training data, the model would correctly classify it. 
The state-of-the-art Truth Serum attack \citep{TruthSerum}, which requires access to model confidences, employs a static poisoning strategy by adding a fixed set of  $k$ poisoned replicas for each challenge point. 
This poisoning strategy  fails for label-only MI as often times both IN and OUT models misclassify the target sample, 
%As a result adapting their strategy in our threat model proves to be unsuccessful, 
as discussed in Section \ref{sec:background}. 
Our \emph{crucial} observation is that not all challenge points require the same number of poisoned replicas to create a separation between IN and OUT models.  To provide evidence for this insight, we show a visual illustration of model confidences under the same number of poisoned replicas for two challenge points in Figure \ref{fig:AP-Conf}.
%we provide a visual illustration of how adding the same number of poisoned replicas has a disparate impact on the confidence distribution of the two challenge points.
%
Therefore, we propose a new strategy that adaptively selects the number of poisoned replicas for each challenge point, with the goal of creating a separation between IN and OUT models. The IN models trained with the challenge point in the training set should classify the point correctly, while the OUT models should misclassify it. Our strategy adaptively adds poisoned replicas until the OUT models consistently misclassify the challenge point at a significantly higher rate than the IN models.

% Our strategy adaptively adds poisoned replicas until the target point is mis-predicted by the OUT models, at which point the IN models  still classify the point correctly most of the times. 
%Our strategy involves constructing a poisoned set by adding poisoned replicas until the model's confidence on the given challenge point drops below a predefined threshold $\psnthresh$.

%Once the poisoned set is generated, we can employ a decision boundary-based approach for our attack, similar to previous works \citep{choquette-choo21a,LO2}. However, this approach requires a substantial number of queries to the target model and we propose a more efficient alternative.
Existing MI attacks with poisoning \citep{TruthSerum,MIPoison22} utilize confidence scores obtained from the trained model to build a distinguishing test. As our attacker only obtains the predicted labels, we develop a label-only ``proxy'' metric for estimating the model's confidence, by leveraging the predictions obtained on ``close'' neighbors of the challenge point. 
%Previous works \citep{choquette-choo21a,LO2} also propose a (heuristic) data augmentation based approach, but lack a clear understanding of which subset of augmentations are truly beneficial for the distinguishing test. 
We introduce the concept of a \emph{membership neighborhood}, which is constructed by selecting the closest neighbors based on the KL divergence computed on model confidences. This systematic selection helps us improve the effectiveness of our attack by strategically incorporating only the relevant neighbor predictions. The final component of the attack is the distinguishing test,
%After obtaining labels on the challenge point and its neighbors, we proceed to the final component of our attack, which involves performing the distinguishing test. 
in which we compute a score based on the target model $\tgtmodel$'s correct predictions on the membership neighborhood set. These scores are  used to compute the TPR at fixed FPR values, as well as other metrics of interest such as AUC and MI accuracy. 
%similar to prior works \citep{LiRA,TruthSerum,LO2,wen2023canary}.
%
% Towards this, we analyze the distribution of the misclassifcation rate of the poisoned models with and without the presence of the challenge point in the training data. Our analysis reveals that under poisoning these distributions follow a "Add name" distribution.   We leverage this observation, and  use a Likelihood-ratio test to determine the distribution in which the challenge point is more likely to belong.
%
We provide a detailed description for each stage of our attack below.

\subsection{Attack Details} 
\label{sec:attack_details}
Our \cham attack can be described as a three-stage process:
%Our attack begins with the adaptive poisoning stage, where the attacker creates a poisoned set by strategically introducing a varying number of poisoned samples  tailored to each challenge point. In the next stage the attacker determines the closest neighbors for each challenge point within the membership neighborhood. Finally, the attacker constructs a distinguishing test by querying the black-box target model using the neighbors, ultimately succeeding in the privacy game.  We provide details for each stage below:

% Finally, using these neighbors, the attacker constructs a distinguishing test after the model owner trains their target model on the poisoned dataset.

 \begin{wrapfigure}{r}{0.33\textwidth} %this figure will be at the right
     \centering
     % \vspace{-12mm}
     \includegraphics[width=0.33\textwidth]{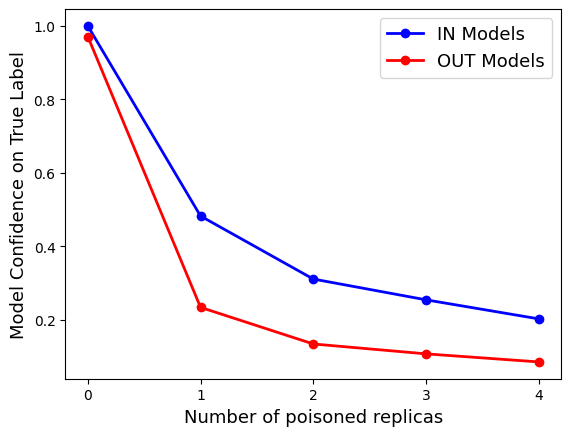}
     \caption{Impact of poisoning on  confidences of IN and OUT models (wrt. the true label) for a challenge point in CIFAR-10 dataset.}
     \vspace{-2mm}
     \label{fig:InOutConfs}    
 \end{wrapfigure}

\paragraph{Adaptive Poisoning.}
% \paragraph{Adaptive Poisoning.}
Given a challenge point $(x,y)$ and access to the underlying training distribution $\dist$, the attacker constructs a training dataset $\advdata \sim \dist$, such that $(x,y) \notin \advdata$ and  a poisoned replica $(x,y')$, for some label $y' \neq y$.
%
% the attacker constructs a poisoned replica $(x,y')$, for some label $y' \neq y$, along with a training dataset $\advdata \sim \dist$, such that $(x,y) \notin \advdata$. 
%
The goal of the attacker is to construct a small enough poisoned set $\psndata$ such that a model trained on $\advdata \cup \psndata$, which \emph{excludes} $(x,y)$, missclassifies the challenge point. The attacker needs to train their own OUT shadow models (without the challenge point) to determine how many poisoned replicas are enough to  mis-classify the challenge point. Using a set of $m$ OUT models instead of a single one increases the chance that any other OUT model (e.g., the target model $M_t$) has a similar behavior under poisoning. The attacker begins with no poisoned replicas and trains $m$ shadow models on the training set $\advdata$. The attacker adds a poisoned replica if the average confidence across the OUT models on label $y$ is above a threshold $\psnthresh$, and repeats the process until the models' average confidence on label $y$ falls below  $\psnthresh$ (threshold where mis-classification for the challenge point occurs). The details of our adaptive poisoning strategy are outlined in Algorithm \ref{alg:adp_psn}, which describes the iterative procedure for constructing the poisoned set. 

% \vspace{-1mm}

Note that we  \emph{do not} need to separately train any IN models, i.e., models trained on $\psndata \cup \advdata \cup \{(x,y)\}$, to select the number of poisoned replicas for our challenge point. This is due to our observation that, in presence of poisoning, the average confidence for the true label $y$ tends to be higher on the IN models when compared to the OUT models. 
Figure \ref{fig:InOutConfs} illustrates an instance of this phenomenon, where the average confidence on the OUT models decreases at a faster rate than the confidence on the IN models with the addition of more poisoned replicas.
As a result, the confidence mean computed on the OUT models (line 6 in Algorithm \ref{alg:adp_psn}) will always cross the poisoning threshold $\psnthresh$ first, leading to  misclassification of the challenge point by the OUT models before the IN models. Therefore, we are only required to train OUT models for our adaptive poisoning strategy.

\begin{algorithm}[h]
	\begin{algorithmic}
		
        \smallskip
		\State {\bf Input:} Challenge point $(x,y)$, poisoned point $(x, y')$ where $y' \neq y$, attacker's dataset $\advdata$, poison threshold $\psnthresh$ and maximum poisoned iterations $\mathsf{k_{max}}$.

        \smallskip
        \State 1: Let $k$ denote the number of poisoned replicas.

		%\noindent
		%\Comment{Offline Phase}
		%\State  1. Construct poisoned dataset $D_p$ of size $p\cdot m$ and forward  it to the model owner.
       % \State 2:  {\bf while} $\mathsf{True}$ {\bf do:}
       % \smallskip
       \State 2:  {\bf For} $k = 0,\ldots, \mathsf{k_{max}}$ {\bf do:}
       
		%	 \State Compute $T = \frac{\sum_{i=1}^{k} y_i}{k} $

            % \smallskip
            \State 3:\indent Construct poisoned dataset $\psndata$ containing $k$ replicas of $(x, y')$.
  
		% \smallskip
		\State   4:\indent Train $m$ OUT models $\{\theta^{\text{out}}_1,\ldots, \theta^{\text{out}}_m\}$ on  dataset $ \psndata  \cup \advdata$. 
		
		\State 5:\indent Query $x$ on OUT models and obtain confidences ${c^y_1,\ldots, c^y_m}$ for label $y$, where $0 \leq c^y_i \leq 1$.

            \State 6:\indent Compute  mean of the confidences $\mu = \frac{\sum_{i=1}^{m} c^y_i}{m}$.
  
            % \smallskip
		\State  7:\indent {\bf If $\mu \leq \psnthresh$:} 

            \State 8:\indent \indent {\bf break}

            \State 9:\indent $k = k+1$

        \smallskip
        \State {\bf Output:} Number of poisoned replicas $k$.    % \smallskip
		
	\end{algorithmic}
	\caption{Adaptive Poisoning Strategy}
	\label{alg:adp_psn}
\end{algorithm}

In practical scenarios, an attacker would aim to infer membership across multiple challenge points rather than focusing on a single point.  Later in \Cref{sec:MultPoint}, we propose a strategy that handles a set of challenge points simultaneously while naturally capturing interactions among them during the poisoning phase. Importantly, our strategy only incurs a fixed overhead cost, enabling the attacker to scale to any number of challenge points.  

\paragraph{Membership Neighborhood.} 
% \paragraph{Membership Neighborhood.}
In this stage, the attacker's objective is to create a membership neighborhood set $\nbrset{(x,y)}$ by selecting close neighboring points to the challenge point. This set is then used to compute a proxy score in order to build a distinguishing test.
%
% After constructing the poisoned set $\psndata$ using our strategy, the attacker forwards $\psndata$ to the challenger, as described in Section \ref{sec:threatmodel}. 
%
To construct the neighborhood set, the attacker needs $N$ shadow models such that the challenge point $(x,y)$ appears in the training set of half of them (IN models), and not in the other half (OUT models). Interestingly, the attacker can reuse the OUT models trained from the previous stage and reduce the computational cost of the process. 
% Note that, $\psndata$ is also included in each shadow model's training set.
Using these shadow models, the attacker constructs the neighborhood set $\nbrset{(x,y)}$ for a given challenge point $(x, y)$. A candidate $(x_c,y)$, where $x_c \neq x$, is said to be in set $\nbrset{(x,y)}$, if the original point and the candidate's model confidences are close in terms of KL divergence for both IN and OUT models, i.e., the following conditions are satisfied:
\begin{align}
 \label{eq:neighbor}
 \mathsf{KL}(~\Phi(x_c)_{\textsf{IN}} ~|| ~\Phi(x)_{\textsf{IN}}~) \leq  \nbrthresh
 ~\mathsf{and}~ \mathsf{KL}(~\Phi(x_c)_{\textsf{OUT}}~||~ \Phi(x)_{\textsf{OUT}}~) \leq  \nbrthresh
\end{align}
Here, $\mathsf{KL}()$ calculates the Kullback-Leibler divergence between two distributions. Notations $\Phi(x_c)_{\textsf{IN}}$ and $\Phi(x_c)_{\textsf{OUT}}$ represent the distribution of confidences (wrt. label $y$) for candidate $(x_c, y)$ on the IN and OUT models trained with respect  to challenge point $(x,y)$. 

Note that the models used in this stage do not need to include poisoning into their training data. We observe that the distribution of confidence values for candidates characterized by low KL divergence tend to undergo similar changes as those of the challenge point when poisoning is introduced.
% We observe that candidates with close confidence distribution to the challenge point prior to poisoning follow a similar change in model confidences to the challenge point when poisoning is introduced.  
%We observe candidates that exhibit a comparable confidence distribution to the challenge point (based on Equation (\ref{eq:neighbor})) prior to the introduction of poisoning also demonstrate a similar change in the behavior of confidence distributions as the challenge point when poisoning is introduced. 
We call such candidates  \emph{close neighbors}. In Figure \ref{fig:Mem-Nbd}, we show how the confidence distribution of a close neighbor closely mimics the confidence distribution of the challenge point as we add  two poisoned replicas. Additionally, we also demonstrate that the confidence distribution of a \emph{remote neighbor} is hardly affected by the addition of  poisoned replicas and does not exhibit a similar shift in its confidence distribution as the challenge point. Therefore, it is enough to train shadow models without poisoning, which reduces the time complexity of the attack.

\begin{figure}[h] %this figure will be at the right
{    \centering
    \includegraphics[width=\textwidth]{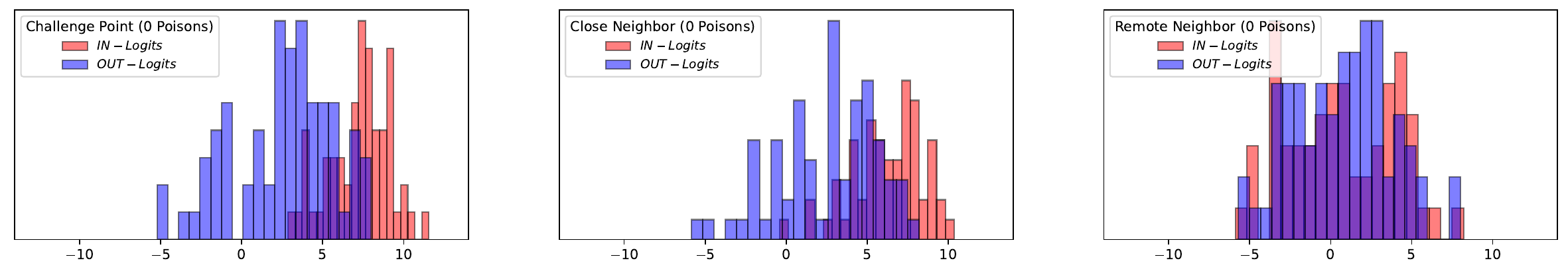}
    \includegraphics[width=\textwidth]{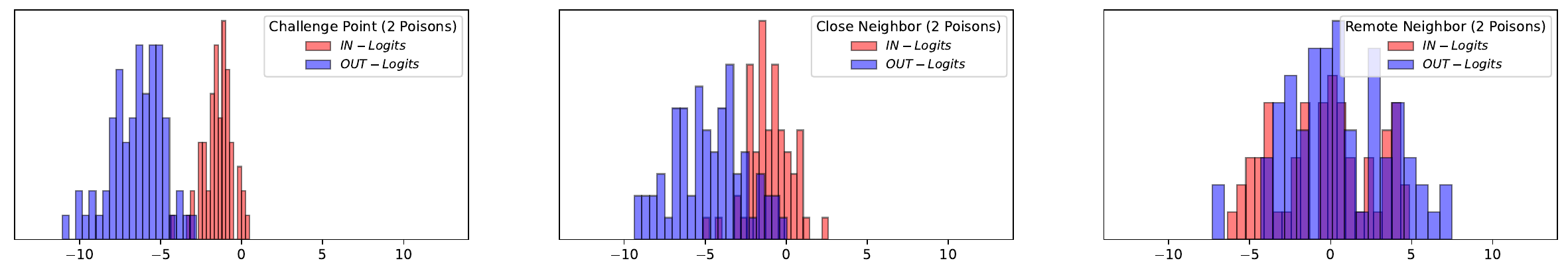}
    \caption{Effect of poisoning on the scaled confidence (logit) distribution of a challenge point and its neighbors. Both the IN and OUT distributions of the near neighbor, unlike the far-away neighbor, exhibit a behavior similar to the challenge point distribution before and after introduction of poisoning.}
    \vspace{-2mm}
    \label{fig:Mem-Nbd}
}    
\end{figure}

In practical implementation, we approximate the distributions $\Phi(x_c)_{\textsf{IN}}$ and $\Phi(x_c)_{\textsf{OUT}}$ using a scaled version of confidences known as logits. Previous work \citep{LiRA,TruthSerum}   showed that logits exhibit a Gaussian distribution, and therefore we  compute the KL divergence between the challenge point and the candidate confidences using Gaussians.  In Section \ref{sec:Ablation}, we empirically show the importance of selecting close neighbors.
% and its significant impact on the effectiveness of our attack.

\paragraph{Distinguishing Test.}
% \paragraph{Distinguishing Test}
The final goal of the attacker is to perform the distinguishing test. Towards this objective, the attacker queries the black-box trained model $M$ using the challenge point and its neighborhood set $\nbrset{(x,y)}$ consisting of $n$ close neighbors.
%
% Once the membership neighborhood for challenge point $(x,y)$ is constructed, the attacker proceeds to query the black-box target model using the challenge point and its neighborhood set $\nbrset{(x,y)}$. 
The attacker obtains a set of predicted labels $\{\hat{y}_1,\ldots,\hat{y}_{n+1}\}$ in return and computes the missclassification score of the trained model $f(x)_y =\frac{\sum_{i=1}^{n+1} \hat{y_i} \neq y} {n+1}$. The score $f(x)_y$ denotes the fraction of neighbors whose predicted labels do not match the ground truth label. This score is then used to predict if the challenge point was a part of the training set or not. Correspondingly, we use the computed misclassification score $f(x)_y$ to calculate various metrics, including TPR@ fixed FPR, AUC, and MI accuracy.

%similar to model confidences used by prior works \citep{LiRA, MIPoison22,TruthSerum, wen2023canary}
% To determine whether $(x,y)$ is a member, the attacker compares the misclassification score $f(x)_y$ to a threshold $\tau$. If $f(x)_y \leq \tau$, the attacker predicts $(x,y)$ as a member.
%
% The threshold $\tau$ can be tuned by training IN and OUT models and computing the misclassification scores on these models.  Interestingly, the set of IN and OUT shadow models that were originally used to construct the membership neighborhood can be reused here to tune the threshold $\tau$. This eliminates the need for training any additional set of shadow models, further streamlining the process.
 
\input{Analysis}

%% file: Analysis.tex
\subsection{ Label-Only MI Analysis} \label{sec:Analysis}

\begin{wrapfigure}{r}{0.4\textwidth} %this figure will be at the right
    \centering
    \vspace{-3mm}
    \includegraphics[width=0.4\textwidth]{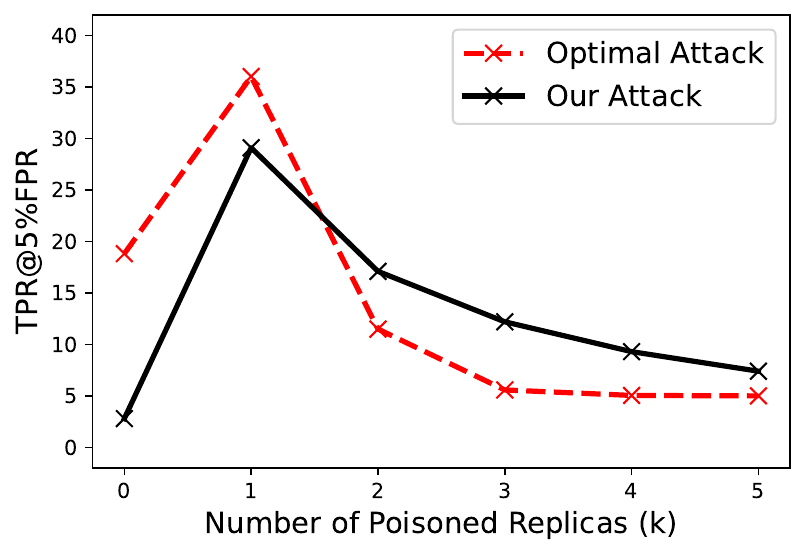}
    \caption{Comparing  Theoretical and Practical attack under poisoning.}
    \vspace{-2mm}
    \label{fig:Theoretical_TPR}    
\end{wrapfigure}

We now analyze the impact of poisoning on MI leakage in the label-only setting.
We construct an \emph{optimal} attack that maximizes the True Positive Rate (TPR) at a fixed FPR of $x\%$, when $k$ poisoned replicas related to a challenge point are introduced in the training set. 
The formulation of this optimal attack is based on a list of assumptions outlined in Appendix \ref{apdx:MIAnalysis}. The objective of constructing this optimal attack is twofold.
First, we aim to examine how increasing the number of poisoned replicas influences the maximum TPR (@$x\%$FPR). Second, we seek to evaluate whether the behavior of \atkname\ aligns with (or diverges from) the behavior exhibited by the optimal attack.
In Figure \ref{fig:Theoretical_TPR}, we present both attacks at a FPR of 5\% on CIFAR-10. The plot depicting the optimal attack shows an initial increase in the maximum attainable TPR following the introduction of poisoning. However, as the number of poisoned replicas increases, the TPR decreases, indicating that excessive poisoning in the label-only scenario adversely impacts the attack TPR. Notably, \atkname\ exhibits a comparable trend to the optimal attack, showing first an increase and successively a decline in TPR as poisoning increases. This alignment suggests that our attack closely mimics the behavior of the optimal attack and has a similar decline in TPR due to overpoisoning. The details of our underlying assumptions and the optimal attack are given in Appendix \ref{apdx:MIAnalysis}. 

%% file: multipoint_alg.tex
\section{Handling Multiple Challenge Points} \label{sec:MultPoint}

\begin{algorithm}[t]
	\begin{algorithmic}
		\smallskip
		\State {\bf Input:} Dataset $\advdata$, set of challenge points $S_c = \{(x_1,y_1), \ldots, (x_n,y_n)\} \subset \advdata$, set of poisoned points $ S_p = \{(x_1, y'_1),\ldots,(x_n,y'_n)\}$ such that  $ \forall_{i \in [1,n]} y'_i \neq y_i$,
        % attacker's dataset $\advdata$ where $\advdata \cap S_c = \oldemptyset$, 
        poison threshold $\psnthresh$ and maximum iterations $\mathsf{k_{max}}$, $m$ number of OUT models to train.
            
            \smallskip
            \State ~~1: Set of poisoned replicas counters $S_k = \{k_1 = 0, \ldots, k_n =0\}$. 
            \smallskip
            \State ~~2: Initialize a set $\{b_1=0,\ldots, b_n = 0 \}$ to keep track of break condition.
		\smallskip
            \State ~~3: Construct subsets $\textsf{D}_i$, with $i \in [1, 2m]$, by randomly sampling half of $\advdata$.
            
		\smallskip
            \State ~~4:  {\bf For} $k = 0,\ldots, \mathsf{k_{max}}$ {\bf do:}
       
            \smallskip
            \State ~~5:\indent Construct poisoned dataset $\psndata$ such that $ \forall_{i \in [1,n]}$ $\psndata$ contains $k_i$ replicas of $(x_i, y'_i)$.

            % \State ~~5:\indent Train $m$ OUT models $\{\theta^{\text{out}}_1,\ldots, \theta^{\text{out}}_m\}$ on  dataset $ \psndata  \cup \advdata$.           
            \smallskip
		\State ~~6:\indent Train $2m$ models $\{\theta_1,\ldots, \theta_{2m}\}$ on  dataset  $ \psndata  \cup \textsf{D}_{i \in [1, 2m]}$
  
            \smallskip
            \State ~~7:\indent {\bf For} $i = 1,\ldots,n$ {\bf do:}

            \smallskip
            \State ~~8:\indent\indent OUT $\leftarrow$ subset of $m$ models whose training data did not include challenge point $i$

            \smallskip
		\State ~~9:\indent\indent Query $x_i$ on the OUT models and obtain model confidences $\{c^{y_i}_1,\ldots, c^{y_i}_m\}$
            
            % \smallskip
            \State ~~10:\indent\indent Compute the mean $\mu_i = \frac{\sum_{j=1}^{m} c^{y_i}_j}{m}$.
  
            \smallskip
		\State ~~11:\indent\indent {\bf If $\mu_i \leq \psnthresh$:} 

            \smallskip
            \State 12:\indent \indent\indent Set $b_i = 1$.

            % \smallskip
            \State 13:\indent\indent {\bf Else:}

            \smallskip
            \State 14:\indent \indent\indent  $k_i = k_i+1$.

            \smallskip
            \State  15:\indent {\bf If $\sum_{i=1}^{n} b_i = n$ :} 

            \smallskip
            \State 16:\indent \indent {\bf break}

            % \smallskip

        \smallskip
        \State {\bf Output:} Set of number of poisoned replicas $S_k$.    % \smallskip
		
	\end{algorithmic}
	\caption{Adaptive poisoning strategy on a set of challenge points}
	\label{alg:adp_psn_set}
\end{algorithm}

\paragraph{Adaptive Poisoning Strategy.} {
We now explore how to extend Algorithm \ref{alg:adp_psn} from Section \ref{sec:Intuit} to infer membership on a set of $n$ challenge points. One straightforward extension is applying Algorithm \ref{alg:adp_psn} separately to each of the $n$ challenge points, but there are several drawbacks to this approach. First, the number of  shadow models  grows proportionally to the number of challenge points $n$, rapidly increasing the cost and impracticality of our method. Second, this method would introduce poisoned replicas by analyzing each challenge point independently, overlooking the potential influence of the presence or absence of other challenge points and their associated poisoned replicas in the training dataset.} 

{Consequently, we propose Algorithm \ref{alg:adp_psn_set}, which operates over multiple challenge points by training a fixed number of shadow models. In Algorithm \ref{alg:adp_psn_set} (Step 3), we start with constructing $2m$ subsets by randomly sampling half of the original training set. These subsets provide various combinations of challenge points, depending on their presence or absence in the subset, helping us tackle the second drawback.  We then use these $2m$ subsets to iteratively construct the  poisoned set $\psndata$ and train $2m$ shadow models per iteration (Steps 4-16).  Thus, the total number of shadow models trained over the course of Algorithm \ref{alg:adp_psn_set} is $2(\mathsf{k_{max}}+1)m$, where $m$ and $\mathsf{k_{max}}$ are hyperparameters chosen by the attacker and not dependent on the number of challenge points $n$.  In fact, the cost of Algorithm \ref{alg:adp_psn_set} is only  a constant factor $2\times$ higher than Algorithm \ref{alg:adp_psn} that was originally designed for a single challenge point.
Later in Appendix \ref{apndx:CostAnalysis} we analyze the impact on our attack's success by varying hyperparameters $m$ and $\mathsf{k_{max}}$.}

Later, in \Cref{apndx:CostAnalysis}, we perform a detailed analysis on the computational cost of our attack by varying the parameters $\mathsf{k_{max}}$ and $m$ to examine their impact on our attack's success.

%For resource-limited scenarios, we also propose an optimized variant of our attack requiring just $16$ shadow models while outperforming prior label-only attacks. A detailed description of our approach and cost analysis is reported in \Cref{sec:MultPoint} and \ref{apndx:CostAnalysis}, respectively.

\paragraph{Membership Neighborhood.}
Recall that, to construct the membership neighborhood for each challenge point, the attacker needed both IN and OUT shadow models specific to that challenge point. However by design of Algorithm \ref{alg:adp_psn_set}, we can  now repurpose the $2m$ models ($m$ IN and $m$ OUT) trained during the adaptive poisoning stage to build the neighborhood. Consequently, there is no need to train any additional shadow models, making this stage very efficient.

%% file: experiments.tex
\section{Experiments} \label{sec:experiments}
We  show that  \cham  significantly improves upon prior label-only MI, then we perform several ablation studies, and finally we evaluate if differential privacy (DP) is an effective mitigation. 
%
% We evaluate the performance of our \cham attack on four different datasets: three computer vision datasets (GTSRB, CIFAR-10 and CIFAR-100) and one tabular dataset (Purchase-100). 
%
%Further, we systematically vary the parameters at each stage of our attack to analyze their individual impact on the attack's performance. Finally, we present our attack's performance on different model architectures and other data modalities such as tabular data. 

\subsection{Experimental Setting} \label{sec:exp-setup}
% We follow a setting similar to Tram\`er et al. (2022) for our experiments on... 
We perform experiments on four different datasets: three computer vision datasets (GTSRB, CIFAR-10 and CIFAR-100) and one tabular dataset (Purchase-100).  
%We provide the experimental details for our vision datasets in this section and defer the details for the tabular data in Section \ref{sec: DMandA}.
We use a ResNet-18 convolutional neural network model for the vision datasets.  We follow the standard training procedure used in prior works \citep{LiRA,TruthSerum,wen2023canary}, including weight decay and common data augmentations for image datasets, such as random image flips and crops. Each model is trained for 100 epochs, and its training set is constructed by randomly selecting 50\% of the original training set. 

To instantiate our attack, we pick 500 challenge points at random from the original training set. In the adaptive poisoning stage, we set the poisoning threshold $\psnthresh = 0.15$, the number of OUT models $m  = 8$ and the number of maximum poisoned iterations $\mathsf{k_{max}} = 6$. In the membership neighborhood stage, we set the neighborhood threshold $\nbrthresh = 0.75$, and the size of the neighborhood $\abs{\nbrset{(x,y)}} = 64$ samples. Later in Section \ref{sec:Ablation}, we vary these parameters and explain the rationale behind selecting these values. 
To construct neighbors in the membership neighborhood, we generate a set of random augmentations for images and select a subset of $64$ augmentations that satisfy Eqn. (\ref{eq:neighbor}).
Finally we test our attack on $64$ target models,  trained using the same  procedure,  including the poisoned set. Among these, 32 serve as IN models, and the remainder as OUT models, in relation to each challenge point. Therefore, the evaluation metrics used for comparison are computed over 32,000 observations.

\paragraph{Evaluation Metrics.}
% Consistent with prior work \citep{LiRA,MIPoison22,TruthSerum, wen2023canary}, our evaluation primarily focuses on two key metrics: TPR @1\% FPR (True Positive Rate when the False Positive rate is $1\%$) and the AUC (Area Under the Curve) score of the ROC (Receiver Operating Characteristic) curve. 
Consistent with prior work \citep{LiRA,MIPoison22,TruthSerum, wen2023canary}, our evaluation primarily focuses on True Positive Rate (TPR) at various  False Positive Rates (FPRs) namely 0.1\%, 1\%, 5\% and 10\%. 
To provide a comprehensive analysis, we also include the AUC (Area Under the Curve) score of the ROC (Receiver Operating Characteristic) curve and  the Membership Inference (MI) accuracy when comparing our attack with prior label-only attacks \citep{yeom2018privacy, choquette-choo21a, LO2}. 
%However, note that MI accuracy alone is inadequate for measuring the performance of MI attacks, as privacy is not an average case metric, as highlighted by \citet{StienkeJon20} and \citet{LiRA}.

\subsection{\cham attack improves Label-Only MI}

We compare our attack against two prior label-only attacks: the Gap attack \citep{yeom2018privacy}, which predicts any misclassified data point as a non-member and the state-of-the-art Decision-Boundary attack \citep{choquette-choo21a,LO2}, which uses a sample's distance from the decision boundary to  determine its membership status.  The Decision-Boundary attack relies on black-box adversarial example attacks \citep{Brendel18,HopSkipJump20}. Given a challenge point $(x,y)$, the attack starts from a random point $x'$ for which the model's prediction is \emph{not} label $y$ and walks along the boundary while minimizing the distance to $x$. The perturbation needed to create the adversarial example estimates the  distance to the decision boundary, and a sample is considered to be in the training set if the estimated distance is above a threshold, and outside the training set otherwise. \citet{choquette-choo21a} showed that their process closely approximates results obtained with a stronger white-box adversarial example technique~\citep{CW18} using $\approx$ 2,500 queries per challenge point. Consequently, we directly compare with the stronger white-box version and show that our attack outperforms even this upper bound.

{\begin{table}[h!]
		\centering 
            \caption{ {\bf Comparing Label-only attacks} on GTSRB (G-43), CIFAR-10 (C-10) and CIFAR-100 (C-100) datasets. Our attack achieves high TPR across various FPR values compared to prior attacks.} \label{tab:main_res}
		\begin{adjustbox}{max width=\textwidth}{  
				\begin{tabular}{l c c c r r r r r r r r r} 
					
					%\toprule
                    % & \multirow{4}{*}{\rotatebox[origin=c]{0}{Poison}} & \multirow{4}{*}{\rotatebox[origin=c]{0}{SMs}} & & &  & \\ 
				
                     & \multicolumn{3}{c}{\bf TPR@0.1\%FPR} & \multicolumn{3}{c}{\bf TPR@1\%FPR}  & \multicolumn{3}{c}{\bf TPR@5\%FPR} & \multicolumn{3}{c}{\bf TPR@10\%FPR}\\ 

					 \cmidrule(lr){2-4} \cmidrule(lr){5-7} \cmidrule(lr){8-10} \cmidrule(lr){11-13}
                    
                    {\bf Label-Only Attack} & G-43 & C-10 & C-100 & G-43 & C-10 & C-100 & G-43 & C-10 & C-100 & G-43  & C-10 & C-100 \\

                    \midrule

                     Gap  & 0.0\% & 0.0\% & 0.0\% &  0.0\% & 0.0\% & 0.0\% & 0.0\% & 0.0\% & 0.0\%  & 0.0\% & 0.0\% & 0.0\% \\

                      Decision-Boundary & 0.04\% & 0.08\% & 0.02\%  & 1.1\% & 1.3\% & 3.6\% &  5.4\% & 5.6\% & 23.0\% & 10.4\% & 11.6\% & 44.9\%  \\

                      \midrule
                      
                      {\bf \atkname~(Ours)}  & {\bf 3.1\%} & {\bf 8.3\%} & {\bf 29.6\%} & {\bf 11.4\%} & {\bf 22.8\%} & {\bf 52.5\%} &  {\bf 25.9\%} & {\bf 34.7\%} & {\bf 70.9\%} & {\bf 35.0\%} & {\bf 42.8\%} & {\bf 79.4\%}\\

                    % Gap  & \tikz\draw[black, thick] (0,0) circle (.6ex); & \tikz\draw[black, thick] (0,0) circle (.6ex);& \tikz\draw[black, thick] (0,0) circle (.6ex);&  0.0\% & 0.0\% & 0.0\% & 50.6\% & 57.7\% & 73.8\%  & 50.6\% & 57.7\% & 73.8\% \\
                    
                    % Decision-Boundary  & \tikz\draw[black, thick] (0,0) circle (.6ex); & \tikz\draw[black,  fill = black, thick] (0,0) circle (.6ex);& \tikz\draw[black,  fill = black, thick] (0,0) circle (.6ex); & 1.1\% & 1.3\% & 2.4\%  & 51.5\% & 62.8\% & 84.1\%   & 51.3\% & 62.4\% & 81.2\%\\

                    % \midrule

                    % Chameleon (Ours) & \tikz\draw[black, fill = black, thick] (0,0) circle (.6ex); & \tikz\draw[black,  fill = black, thick] (0,0) circle (.6ex);& \tikz\draw[black,  fill = black, thick] (0,0) circle (.6ex); & {\bf 11.8\%} &  {\bf 22.9\%} & {\bf 51.5\%} & {\bf 71.7\%} & {\bf 75.9\%} & {\bf 92.3\%}  & {\bf 65.2\%} & {\bf 68.5\%} & {\bf 85.2\%} \\

				\end{tabular}
			}
		\end{adjustbox}
	\end{table}
}

Table \ref{tab:main_res} provides a detailed comparison of \atkname~with prior label-only MI attacks.
% over three datasets. We evaluate each attack based on the TPR at four different FPR values.
\atkname~shows a significant TPR improvement over all FPRs compared to prior works. In particular, for the case of  TPR at $0.1\%$ FPR, prior works achieve TPR values below $0.08\%$, but \cham achieves TPR values ranging from $3.1\%$ to $29.6\%$ across the three datasets, marking a substantial improvement ranging from $77.5\times$ to $370\times$. At $1\%$ FPR, the TPR improves by a factor between $10.36\times$ and $17.53\times$.  Additionally, our attack consistently surpasses prior methods in terms of AUC and MI accuracy metrics. A detailed comparison can be found in Table \ref{tab:MIAUC_res} (Appendix \ref{apndx:MIandAUC}).  Notably, \atkname~is significantly more query-efficient, using only $64$ queries to the target model, compared to the decision-boundary attack, which requires $\approx$ 2,500 queries for the MI test, making our attack approximately $39\times$ more query-efficient.

% Table \ref{tab:main_res} provides a detailed comparison of our attack with these two prior label-only MI attacks over three datasets. We evaluate each attack based on the TPR (@1$\%$ FPR), AUC, and MI accuracy metrics. Our attack shows a significant improvement over all three metrics compared to prior works. In particular, when considering the TPR at $1\%$ FPR, prior works achieve TPR values ranging from $1.1\%$ to $2.4\%$, while our attack achieves TPR values ranging from $11.8\%$ to $51.5\%$ across the three datasets. This marks a substantial improvement in the TPR, ranging from $10.7\times$ to $17.5\times$. Additionally, compared to the decision-boundary attack that requires $\approx$ 2,500 queries for the MI test, our attack uses only $64$ queries to the target model making our attack $\approx 39\times$ more efficient in the number of queries. 

Furthermore, \atkname~requires adding a relatively low number of poisoned points per challenge point: an average of $3.5$ for GTSRB, $1.4$ for CIFAR-10, and $0.6$ for CIFAR-100 datasets for each challenge point.
% over 500 challenge  points.
This results in a minor drop in test accuracy of less than $2\%$, highlighting the stealthiness of our attack. Overall, the results presented in Table \ref{tab:main_res} show that  our adaptive data poisoning strategy  significantly amplifies the MI leakage in the label-only scenario while having a marginal effect on the model's test accuracy.

% Finally, our attack on average required adding $3.5$, $1.4$ and $0.6$ poisoned points per challenge point for GTSRB , CIFAR-10 and CIFAR-100 datasets respectively which had a minor drop in the test accuracy by $2\%$ making our attack stealthy.

%% file: ablations.tex
\subsection{Ablation Studies} \label{sec:Ablation}
%We now vary one component of our attack at a time while keeping the others fixed to the values specified in Section \ref{sec:exp-setup}. This allows us to isolate each component and analyze its individual impact on the performance of \atkname.

We perform several ablation studies, exploring the effect of the parameters discussed in Section \ref{sec:exp-setup}.

\begin{figure*}[h]{
        \centering
        
		\begin{subfigure}[b]{0.4\textwidth}
			\includegraphics[width= \textwidth]{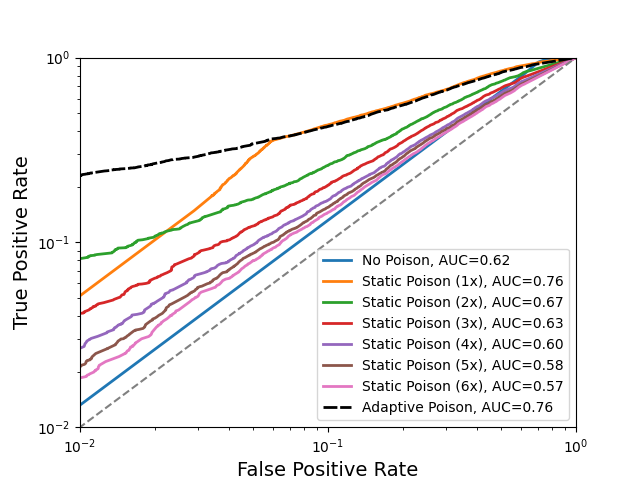}%
			
			\caption{\centering 
			Comparing  adaptive and static poisoning.}
			\label{fig:Ab-AP}
		\end{subfigure}
    	\begin{subfigure}[b]{0.4\textwidth}
                \includegraphics[width=\textwidth]{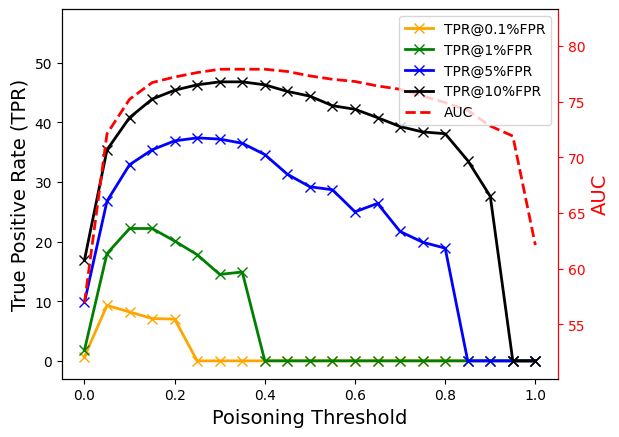}%
			\caption{ \centering TPR/AUC by poisoning threshold $\psnthresh$.}\label{fig:Ab-PT}
		\end{subfigure}		
        	
    	\begin{subfigure}[b]{0.4\textwidth}
			\includegraphics[width=\textwidth]{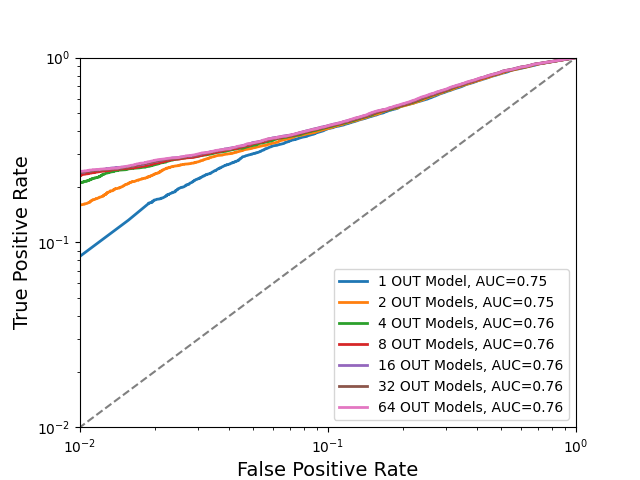}%
			\caption{ \centering TPR/AUC by number of OUT models.}\label{fig:Ab-OUT}
		\end{subfigure}		
            \begin{subfigure}[b]{0.4\textwidth}
			\includegraphics[width=\textwidth]{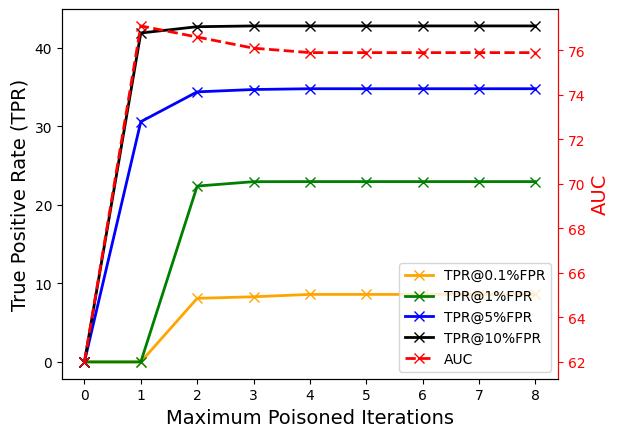}%
			\caption{ \centering TPR/AUC by number of iterations $\mathsf{k_{max}}$.}\label{fig:Ab-MaxRep}
		\end{subfigure}		
        \caption{{\bf Ablations for Adaptive Poisoning stage on CIFAR-10 dataset.} We provide experiments by varying vaious hyperparameters used in the Adaptive Poisoning stage.
        % : poisoning approach, poisoning threshold, number of OUT models, and maximum number of poisoned iterations.
        }
  
}
\end{figure*}

\paragraph{Adaptive Poisoning Stage.}
We evaluate the effectiveness of adaptive poisoning and the impact of several parameters.

\medskip
\noindent
\emph{a) Comparison to Static Poisoning.} Figure \ref{fig:Ab-AP} provides a comparison of our adaptive poisoning approach (Algorithm \ref{alg:adp_psn}) and a static approach where $k$ replicas are added per challenge point. Our approach achieves a TPR@$1\%$ FPR of $22.9\%$, while the best static approach among the six versions achieve a TPR@$1\%$ FPR of $8.1\%$. 
The performance improvement of $14.8\%$ in this metric demonstrates the effectiveness of our adaptive poisoning strategy over static poisoning, a strategy used in Truth Serum for confidence-based MI. Additionally, our approach matches the best static approach (with 1 poison) for the AUC metric, achieving an AUC of $76\%$.
%This also reaffirms our intuitive explanation described in Section \ref{sec:Intuit}, that each challenge point requires a different number of poisoned replicas in order for our attack to succeed in the label-only membership test.
%
% In Figure \ref{fig:Ab-AP} we observe that our adaptive poisoning approach based on Algorithm \ref{alg:adp_psn}, outperforms the static approach where a fixed number of $k$ replicas are added per challenge point. Our approach achieves a TPR (@ $1\%$ FPR) and AUC of $14.2\%$ and $75.9\%$ respectively, significantly outperforming all the static approaches which achieve a TPR (@ $1\%$ FPR) and AUC of at most $2.2\%$ and $69.8\%$ respectively. 

\medskip
\noindent
\emph{b) Poisoning Threshold.} Figure \ref{fig:Ab-PT} illustrates the impact of varying the poisoning threshold $\psnthresh$ (line 7 of Algorithm \ref{alg:adp_psn}) on the attack's performance. When $\psnthresh = 1$, the OUT model's confidence is at most 1 and  Algorithm \ref{alg:adp_psn} will not add any poisoned replicas in the training data, but setting $\psnthresh < 1$ allows the algorithm to introduce poisoning. We observe an immediate improvement in the AUC, indicating the benefits of poisoning over the no-poisoning scenario. Similarly, the TPR improves when $\psnthresh$ decreases, as this forces the OUT model's confidence on the true label to be low, increasing the probability of missclassification. However, setting $\psnthresh$ close to $0$ leads to overpoisoning, where an overly restrictive $\psnthresh$ forces the algorithm to add a large number of poisoned replicas in the training data,  negatively impacting the attack's performance. Therefore, setting  $\psnthresh$ between $0.1$ and $0.25$  results in high TPR. %values, as also seen in Table \ref{tab:main_res} for our other datasets.
% Therefore, assigning $\psnthresh$ to a value between $0.1$ and $0.4$ should result in a high TPR value. 

% We vary the poisoning threshold $\psnthresh$ (line 7 of Algorithm \ref{alg:adp_psn}) from $0$ to $1$ and observe its impact on the attack's performance in Figure \ref{fig:Ab-PT}. Note, setting $\psnthresh = 1$ boils down to the no poisoning case as a model's confidence on a given label will always be $\leq 1$. As we start decreasing $\psnthresh$, which introduces poisoning, we observe an immediate improvement in the AUC indicating poisoning improves upon the non-poisoning scenario. However, we observe that the TPR (@$1\%$ FPR) starts to improve when $\psnthresh$ is set to $ \leq 0.5$, as $\psnthresh \leq 0.5$ forces the model confidence on the true label to be low. Setting $\psnthresh$ closer to $0$, leads to the overpoisoning scenario as a very restrictive $\psnthresh$ forces the algorithm to add a large number of poisoned replicas till the average model confidence is below the poison threshold. 

\medskip
\noindent
\emph{ c) Number of OUT Models.} Figure \ref{fig:Ab-OUT} shows the impact of the number of OUT models (line 5 of Algorithm \ref{alg:adp_psn}) on our attack's performance. As we increase the number of OUT models from $1$ to $8$, the TPR@$1\%$ FPR shows a noticeable improvement from $13.1\%$ to $22.9\%$. However, further increasing the number of OUT models up to $64$ only yields a marginal improvement with TPR increasing to $24.1\%$. Therefore, setting the number of OUT models to $8$ strikes a balance between the attack's success and the computational overhead of training these models.

\medskip
\noindent
\emph {d) Maximum Poisoned Iterations.} In Figure \ref{fig:Ab-MaxRep}, we observe that increasing the maximum number of poisoned iterations $\mathsf{k_{max}}$ (line 2 of Algorithm \ref{alg:adp_psn}) leads to significant improvements in  both the TPR and AUC metrics. When $\mathsf{k_{max}} = 0$, it corresponds to the no poisoning case.
%as Algorithm \ref{alg:adp_psn} does not add any poisoned replicas to the training data. 
The TPR@1\% FPR and AUC metrics improve as parameter $\mathsf{k_{max}}$ increases. However, the TPR stabilizes when $\mathsf{k_{max}} \geq 4$, indicating that no more than 4 poisoned replicas are required per challenge point. 
%As a result, when we set $\mathsf{k_{max}} = 5$ or greater, it yields no improvement in the TPR and AUC, as the algorithm terminates before reaching the maximum iteration limit $\mathsf{k_{max}}$.

\paragraph{Membership Neighborhood Stage.} \label{sec:exp-mem_nbrhood}
Next, we explore the impact of varying the the membership neighborhood size and the neighborhood threshold $\nbrthresh$ individually. For the neighborhood size, we observe a consistent increase in TPR@1\% FPR of 0.6\% as we increase the number of queries from $16$ to $64$, beyond which the TPR oscillates. Thus, we set the neighborhood size at 64 queries for our experiments, achieving satisfactory attack success.
For the neighborhood threshold parameter $\nbrthresh$, we note a decrease in TPR@1\% FPR of 6.2\% as we increase $\nbrthresh$ from 0.25 to 1.75. This aligns with our intuition, that setting $\nbrthresh$ to a smaller value prompts our algorithm to select close neighbors, which in turn enhances our attack's performance. The details for these results are in Appendix \ref{apndx:MemNbrhood}.

\subsection{Other Data Modalities and Architectures} \label{sec: DMandA}

To show the generality of our attack, we evaluate \atkname\ on various model architectures,  including ResNet-34, ResNet-50 and VGG-16. We observe a similar trend of high TPR value at various FPRs. Particularly for VGG-16, which has $12\times$ more trainable parameters than ResNet-18, the attack achieves better  performance than ResNet-18 across all metrics, suggesting that more complex models tend to be more susceptible to privacy leakage.
We also evaluate \atkname~against a tabular dataset (Purchase-100). Once again, we find a similar pattern of consistently high TPR values across diverse FPR thresholds. In fact, the TPR@1\% FPR metric reaches an impressive 45.8\% when tested on a two-layered neural network. The details for the experimental setup and the results can be found in  Appendix \ref{apndx:DMA}.

% The results of this analysis are grouped in \Cref{tab:different_modalities} for ease of consultation.

%% file: DP.tex
\subsection{Does Differential Privacy Mitigate \cham?}
We evaluate the resilience of \atkname~against models trained using a standard differentially private (DP) training algorithm, DP-SGD \citep{Abadi16}.  Our evaluation covers a broad spectrum of privacy parameters, but here we highlight results on $\epsilon = \{ \infty, 100, 4\}$, which  represent no bound, a loose bound and a strict bound on the privacy. At $\epsilon$ as high as 100, we observe a decline in \atkname's performance, with TPR@1\% FPR decreasing from 22.6\% (at $\epsilon = \infty$) to 6.1\%. Notably, in the case of an even stricter $\epsilon = 4$,  we observe that TPR@1\% FPR becomes 0\%, making our attack ineffective. However, it is also important to note that the model's accuracy also takes a significant hit, plummeting from 84.3\% at $\epsilon = \infty$ to 49.4\% at $\epsilon=4$, causing a substantial 34.9\% decrease in accuracy. This trade-off shows that while DP serves as a powerful defense, it does come at the expense of model utility. More comprehensive results using a wider range of $\epsilon$ values can be found in  Appendix \ref{apndx: DP}.

%% file: apndx_ablations.tex
\section{Additional Experiments}

In this appendix we include additional results of our experiments with the \atkname\ attack.

\subsection{AUC and MI Accuracy Metric} \label{apndx:MIandAUC}

We report here a comparison between our attack and the existing state of the art label-only MI attacks, Gap \citep{yeom2018privacy} and Decision-boundary \citep{choquette-choo21a}, on aggregate performance metrics: the AUC (Area Under the Curve) score of the ROC (Receiver Operating Characteristic) curve and the average accuracy.
\Cref{tab:MIAUC_res} shows the values of these metrics on the three image datasets used in previous evaluations.
Interestingly, \atkname\ achieves superior average values in all tested scenarios.

{\begin{table}[h]
		\centering 
            \caption{ {\bf Comparison of Label-only attacks on AUC and Membership Inference (MI) Accuracy metric} for GTSRB, CIFAR-10 and CIFAR-100 datasets. Our attack uses a combination of adaptive poisoning, training shadow models (SMs) and careful selection of multiple queries (MQs) to outperform prior attacks.} \label{tab:MIAUC_res}
		\begin{adjustbox}{max width= 1.0\textwidth}{  
				\begin{tabular}{l c c c r r r r r r} 
					
					%\toprule
                    % & \multirow{4}{*}{\rotatebox[origin=c]{0}{Poison}} & \multirow{4}{*}{\rotatebox[origin=c]{0}{SMs}} & & &  & \\ 
				
                     &   &  &  &  \multicolumn{3}{c}{{\bf AUC}}  &  \multicolumn{3}{c}{{\bf MI Accuracy}}\\ 

					\cmidrule(lr){5-7} \cmidrule(lr){8-10}
                    
                    {\bf Label-Only Attack} & {\bf Poison} & {\bf SMs} & {\bf MQs} & GTSRB  & CIFAR-10 & CIFAR-100 & GTSRB  & CIFAR-10 & CIFAR-100\\

                    \midrule

                    Gap  & \tikz\draw[black, thick] (0,0) circle (.6ex); & \tikz\draw[black, thick] (0,0) circle (.6ex);& \tikz\draw[black, thick] (0,0) circle (.6ex); 
                    & 50.6\% & 57.7\% & 73.8\% & 50.6\% & 57.7\% & 73.8\% \\
                    
                    Decision-Boundary  & \tikz\draw[black, thick] (0,0) circle (.6ex); & \tikz\draw[black,  fill = black, thick] (0,0) circle (.6ex);& \tikz\draw[black,  fill = black, thick] (0,0) circle (.6ex); 
                    & 51.5\% & 62.8\% & 84.9\% & 51.3\% & 62.4\% & 81.1\%\\

                    \midrule

                    Chameleon (Ours) & \tikz\draw[black, fill = black, thick] (0,0) circle (.6ex); & \tikz\draw[black,  fill = black, thick] (0,0) circle (.6ex);& \tikz\draw[black,  fill = black, thick] (0,0) circle (.6ex); 
                    & {\bf 71.9\%} & {\bf 76.3\%} & {\bf 92.6\%} & {\bf 65.2\%} & {\bf 68.5\%} & {\bf 85.2\%} \\

				\end{tabular}
			}
		\end{adjustbox}
	\end{table}
}

\subsection{Membership Neighborhood Stage} \label{apndx:MemNbrhood}

\begin{figure*}[h]{
        \centering
        
		\begin{subfigure}[b]{0.45\textwidth}
			\includegraphics[width= \textwidth]{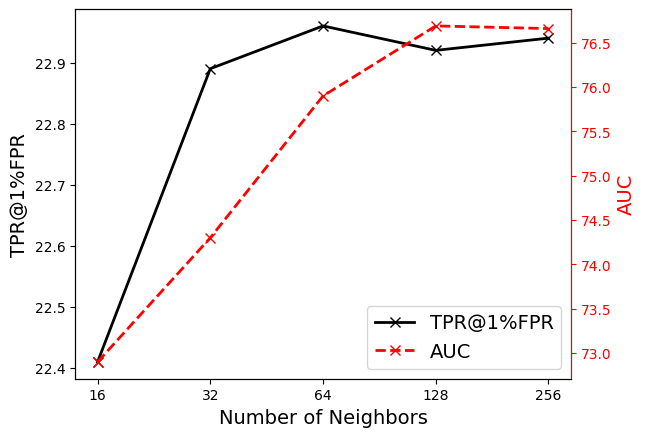}%
			
			\caption{\centering 
			TPR and AUC by varying size of Membership Neighborhood.}
			\label{fig:Ab-Memsize}
		\end{subfigure}
    	\begin{subfigure}[b]{0.45\textwidth}
			\includegraphics[width=\textwidth]{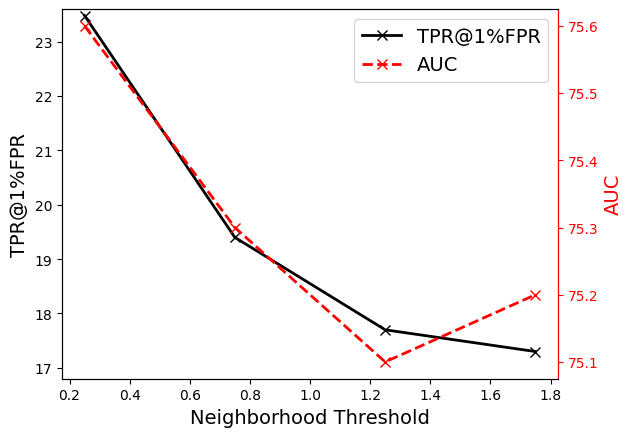}%
			\caption{ \centering TPR  and AUC by varying neighborhood threshold $\nbrthresh$.}\label{fig:Ab-MemTh}
		\end{subfigure}		
        \caption{{\bf Ablations for Membership Neighborhood stage on CIFAR-10 dataset.} We provide experiments on two hyperparameters used in this stage: size of the membership neighborhood, and membership threshold}
  
}
\end{figure*}

We analyze the impact on our attack's performance by varying the size of the membership neighborhood set and the neighborhood threshold $\nbrthresh$ individually. Recall that the membership neighborhood for a challenge point is designed to compute a proxy score for our distinguishing test, which is achieved by finding candidates that satisfy \Cref{eq:neighbor}, described in Section \ref{sec:attack_details}.

\medskip
\noindent
\emph {i) Size of Membership Neighborhood.} In Figure \ref{fig:Ab-Memsize}, we show the impact of the size of membership neighborhood on the attack success. We observe that as the size of the neighborhood increases from $16$ to $64$ samples, the TPR@$1\%$ FPR and AUC of our attack improves by $0.6\%$ and $3\%$ respectively. However, further increasing the neighborhood size beyond $64$ samples does not significantly improve the TPR and AUC, as indicated by the oscillating values. Therefore, setting the membership neighborhood size to $64$ samples or larger should provide a satisfactory level of attack performance.

\medskip
\noindent
\emph{ii) Neighborhood threshold.} We now vary the neighborhood threshold $\nbrthresh$ to observe it's impact on the attack's performance. The neighborhood threshold $\nbrthresh$ determines the selection of neighbors for the challenge point. In Figure \ref{fig:Ab-MemTh}, we observe that setting $\nbrthresh$ to $0.25$ results in the highest TPR@$1\%$ FPR of $23.5\%$, while increasing $\nbrthresh$ to $1.75$ decreases the TPR to $17.3\%$. This aligns with our intuition that systematically including samples in the membership neighborhood that are closer to the challenge point improves the effectiveness of our attack compared to choosing distant samples or using arbitrary random augmentations. Additionally, we observe a decrease in the AUC score when introducing distant samples in the membership neighborhood. However, this decrease is relatively more robust compared to the TPR metric.

\subsection{Data Modalities and Architectures} \label{apndx:DMA}

\begin{table}[h]
\centering
\small
\caption{Effectiveness of \atkname\ on Purchase-100 (P-100) and CIFAR-10 (C-10)  datasets  over various model architectures. Our attack achieves high TPR values across various model architectures.}
\label{tab:different_modalities}
\begin{adjustbox}{max width=\textwidth}{  
\begin{tabular}{lll r r r r r r}

 &  &  & \multicolumn{4}{c}{True Positive Rate (TPR) }&   &  \\ \cmidrule(lr){4-7} 

Modality &  Dataset &  Model Type& @0.1\%FPR & @1\%FPR & @5\%FPR &@10\%FPR &  AUC & MI Accuracy\\ \midrule 

% \multirow{2}{*}{Tabular} & \multirow{2}{*}{P-100} & \multirow{2}{*}{2-NN} & Gap & 0\% & 0\% & 0\% & 56.5\% & 56.5\% \\
 
 Tabular & P-100 & 1-NN & 8.7\% & 34.8\% & 62.2\% & 72.9\% & 92.0\% & 83.8\% \\
        & & 2-NN & 18.3\% & 45.8\% & 76.9\% & 88.8\% & 95.9\% & 89.6\%\\
        % & & 4-NN & 1.8\% & 38.2\% & 69.2\% & 82.9\% & 94.2\% & 87.2\%\\

 \midrule

Image & C-10 & ResNet-18 & 8.2\% & 22.8\% & 34.7\% & 42.8\% &  76.2\% & 68.5\% \\

 &  & ResNet-34 & 9.1\% & 24.6\% & 35.4\% & 43.1\% &  76.6\% & 69.2\% \\

 &  & ResNet-50 & 7.4\%& 23.3\% & 33.9\% & 41.6\% &  76.3\% & 69.8\% \\

 &  & VGG-16 & 8.6\% & 29.4\% & 49.1\% & 59.3\% & 78.6\% & 75.4\% \\
 
 % & & WideResNet-50 & 14.5\% & 32.1\% & 39.4\% & 75.5\% & 68.7\%  \\

\end{tabular}
} \end{adjustbox}
\end{table}

We evaluate here the behavior of our attack on different data modalities, model sizes and architectures.

For our experiments on tabular data, we use the Purchase-100 dataset, with a setup similar to \citet{choquette-choo21a}: a one-hidden-layer neural network, with 128 internal nodes, trained for 100 epochs. 
We do not use any augmentation during training, and we construct neighborhood candidates by flipping binary values according to Bernoulli noise with probability 2.5\%.
Here, we apply the same experimental settings as detailed in \Cref{sec:exp-setup}, with 500 challenges points and $m=8$, and a slightly lower $t_p = 0.1$. 
The first row of \Cref{tab:different_modalities} reports the results of the attack on this modality.
Interestingly, despite the limited size of the model, we observe a high AUC and relatively high TPR at 1\% FPR. 

To observe the effect of \atkname\ for different model sizes we experimented with scaling the dimension of both the feed-forward network used for Purchase-100 and the base ResNet model used throughout \Cref{sec:experiments}.
For the feed-forward model we added a single layer, which was already sufficient to achieve perfect training set accuracy, due to the limited complexity of the classification task.
On CIFAR-10, we compared the results on ResNet 18, 34 and 50, increasing the number of trainable parameters from 11 million to 23 million.
For the larger models we increased the training epochs from 100 to 125.
Finally, we considered a VGG-16~\citep{SimonyanZ14a} architecture pre-trained on ImageNet, which we fine-tuned for 70 epochs. 
This is a considerably larger model, with roughly 138 million parameters.

We observe generally similar attack performance on the ResNet models, as shown in \Cref{tab:different_modalities}.
The largest model for both modalities, instead, showed significantly higher TPR values at low false positive rates such as 0.1\% and 1\%. 
This trend can be attributed to the tendency of larger models to memorize the training data with greater ease.

\subsection{Differential Privacy} \label{apndx: DP}

\begin{table}[h]
\centering
\small
\caption{TPR at various FPR values for our \atkname~ attack when models are trained using DP-SGD on CIFAR-10 dataset. Differential Privacy significantly mitigates the impact of our attack but also adversely impacts the model's accuracy.}
\label{tab:DP}
\begin{adjustbox}{max width=\textwidth}{  
\begin{tabular}{l r r r r r r}

Privacy Budget & Model Accuracy & TPR@1\%FPR & TPR@5\%FPR & TPR@10\%FPR & AUC  & MI Accuracy\\ \midrule 
 % \midrule

 % $\epsilon = 2$ &      -\%    &      -\%     &   -\%        &  -\% & -\% \\

 $\epsilon = \infty $ (No DP) & 84.3\%  &   22.6\%    &      34.8\%     &   42.9\%        &  76.8\% & 69.3\% \\

 $\epsilon = 100$ &  57.6\% &    6.1\%    &      14.8\%     &   26.7\%        &  61.8\% & 59.7\% \\ 

$\epsilon = 50$ &  $56.8$\% &    0.0\%    &      11.7\%     &   21.1\%        &  58.9\% & 57.9\% \\ 

$\epsilon = 32$ &  $57.2$\% &   0.0\%    &      9.9\%     &   18.9\%        &  58.8\% & 58.3\% \\ 

$\epsilon = 16$ &  $56.4$\% &    0.0\%    &     0.0\%     &   14.2\%        &  55.8\% & 56.0\% \\ 

 $\epsilon = 8$ &  $54.8\%$  &  0.0\%    &      0.0\%     &   12.6\%        &  53.7\% & 53.7\% \\

$\epsilon = 4$ &  $49.4\% $  &  0.0\%    &      0.0\%     &   11.1\%        &  52.4\% & 52.6\% \\

\end{tabular}
} \end{adjustbox}
\end{table}

We evaluate the resilience of our \atkname~ attack against models trained using DP-SGD \citep{Abadi16}.  We use PyTorch’s differential privacy library, Opacus \citep{opacus}, to train our models. The privacy parameters are configured with $\epsilon$ values of $\{4,8,16, 32, 50, 100, \infty \}$ and $\delta = 10^{-5}$, alongside a clipping norm of $C=5$. Our training procedure aligns with that of previous works \citep{ImageNetDP22, de2022unlocking}, which involves replacing Batch Normalization with Group Normalization with group size set to  $G = 16$ and the omission of data augmentations, which have been observed to reduce model utility when trained with DP.
In Table \ref{tab:DP}, we observe that as $\epsilon$ decreases, our attack success also degrades showing us that DP is effective at mitigating our attack.  However, we also observe that the accuracy of the model plummets with decrease in $\epsilon$. Thus, DP can be used as a defense strategy, but comes at an expense of model utility.

% In Table \ref{tab:DP}, we observe when $\epsilon$ is set to very low values, for instance $\epsilon= 4$,  significantly impacts our attack success. However the test accuracy of a trained model with $\epsilon= 4$ decreases to $49.4\%$ which is $35\%$ lower compared to a model trained with no DP. 

%% file: apndx_threatmodel.tex
\section{Attack Success and Cost Analysis} \label{apndx:CostAnalysis}

We now analyze the computational cost of our attack and its implications on our attack's success. Recall in \Cref{sec:MultPoint}, we determined the total number of shadow models to be $2(\mathsf{k_{max}}+1)m$, where $m$ and $\mathsf{k_{max}}$ denote the hyperparameters in Algorithm  \ref{alg:adp_psn_set}. In Table \ref{tab:C10-cost}, we vary these parameters and observe their effects on our attack's success and the number of shadow models trained upon the algorithm's completion. Each entry in Table \ref{tab:C10-cost} represents a tuple indicating TPR@1\%FPR and the total shadow models trained, respectively. The results are presented for $\mathsf{k_{max}} \leq 4$, as Algorithm \ref{alg:adp_psn_set} terminates early (at Step 15) when $\mathsf{k_{max}} \geq 5$.

\begin{table}[h]
\centering
\small
\caption{Evaluation of attack success and computational cost for our \atkname~attack on CIFAR-10 dataset by varying hyperparameters $m$ (number of OUT models) and $\mathsf{k_{max}}$ (maximum iterations) given in Algorithm \ref{alg:adp_psn_set}. Each entry is presented as a tuple, indicating TPR@1\%FPR and the total number of (ResNet-18) shadow models trained at the completion of Algorithm \ref{alg:adp_psn_set}.}
\label{tab:C10-cost}
\begin{adjustbox}{min width= 10.5cm}{  
\begin{tabular}{ c | r r r r}

CIFAR-10 &   \multicolumn{1}{c}{$m = 1$} & \multicolumn{1}{c}{$m = 2$} & \multicolumn{1}{c}{$m = 4$} & \multicolumn{1}{c}{$m = 8$} \\
\midrule 
$\mathsf{k_{max}} = 1$ & (0\%, 2) & (0\%, 4) & (1.1\%, 8) & (1.1\%, 16) \\[2pt]

$\mathsf{k_{max}} = 2$ & (13.2\%, 6) & (15.6\%, 12) & (20.7\%, 24) & (22.4\%, 48) \\ [2pt]

$\mathsf{k_{max}} = 3$ & (13.2\%, 8) & (15.8\%, 16) & (20.9\%, 32) & (22.8\%, 64) \\ [2pt]

$\mathsf{k_{max}} = 4$ & (13.6\%, 10) & (15.9\%, 20) & (21.1\%, 40) & (22.9\%, 80) \\ [2pt]

\end{tabular}
} \end{adjustbox}
\end{table}

We observe that our attack attains a notable TPR of $22.9\%$ when setting $m=8$ and $k=4$, albeit at the expense of training $80$ shadow models. However, for the attack configuration with $m=2$ and $k=3$, \atkname\ still achieves a high TPR of $15.8\%$ while requiring only $16$ shadow models. This computationally constrained variant of our attack still demonstrates a TPR improvement of $12.1\times$ over the state-of-the-art Decision Boundary attacks. 
% Additionally, the membership neighborhood stage for selecting close neighbors imposes almost no extra computational cost, as we re-use the models trained in the adaptive poisoning stage. 
Consequently, in computationally restrictive scenarios, a practical guideline would be to set $m=2$ and $k=3$.

Interestingly, even our computationally expensive variant ($m=8$ and $k=4$) still trains fewer models than the state-of-the-art confidence-based attacks like \citet{LiRA, TruthSerum}, which typically use 128 shadow models (64 IN and 64 OUT).

{We also conduct a cost analysis on the more complex CIFAR-100 dataset, as presented in Table \ref{tab:C100-cost}. We observe similar TPR improvement of $13.1\times$ over the Decision-Boundary attack, while training as few as $12$ shadow models. With the CIFAR-100 dataset, our algorithm terminates even earlier at $\mathsf{k_{max}} = 2$, requiring fewer models to be trained.}

\begin{table}[h]
\centering
\small
\caption{{Analyzing \atkname~attack success and computational cost on CIFAR-100, varying hyperparameters $m$ and $\mathsf{k_{max}}$ (Algorithm \ref{alg:adp_psn_set}). Entries represent TPR@1\%FPR and the total number of (ReseNet-18) shadow models trained.}}

\label{tab:C100-cost}
\begin{adjustbox}{min width= 10.5cm}{  
\begin{tabular}{ c | r r r r}

CIFAR-100 &   \multicolumn{1}{c}{$m = 1$} & \multicolumn{1}{c}{$m = 2$} & \multicolumn{1}{c}{$m = 4$} & \multicolumn{1}{c}{$m = 8$} \\
\midrule 
$\mathsf{k_{max}} = 1$ & (33.8\%, 2) & (43.7\%, 4) & (50.3\%, 8) & (51.1\%, 16) \\[2pt]

$\mathsf{k_{max}} = 2$ & (37.2\%, 6) & (47.2\%, 12) & (50.7\%, 24) & (52.5\%, 48) \\ [2pt]

\end{tabular}
} \end{adjustbox}
\end{table}

{\emph{Query Complexity:}  Though prior Decision-Boundary attacks \citep{choquette-choo21a, LO2} do not train any shadow models, they do require a large number of queries (typically 2,500+ queries) to be made to the target model per challenge point. On the contrary, after training $2(\mathsf{k_{max}}+1)m$ shadow models for a set of $n$ challenge points, our attack only requires atmost $64$ queries per challenge point to the target model. This makes our attack $>39\times$ more query-efficient than prior attack.}

{
\emph{Running Time:} We present the average running time for both attacks, on a machine with an AMD Threadripper 5955WX and a single NVIDIA RTX 4090. We run the Decision-Boundary (DB) attack using the parameters provided by the code\footnote{https://github.com/zhenglisec/Decision-based-MIA} in \citet{LO2} on CIFAR-10 dataset. It takes $29.1$ minutes to run the attack on 500 challenge points while achieving a TPR of 1.1\% (@1\%FPR). 
}

{ For our attack, training each ResNet-18 shadow model requires about 80 seconds. That translates to 106.7 minutes of training time for the expensive  configuration of m=8 and k=4  but achieves a substantial TPR improvement of 20.8$\times$ compared to the DB attack. Conversely, the computational restricted version of our attack (m=2 and k=3) takes only 21.3 minutes of training time  while still achieving a significant TPR improvement of 12.1$\times$ over the DB attack.}

% {That translates to 21.3 minutes of training time for a configuration of m=2 and k=3 on CIFAR-10 with 15.8\% TPR@1\%FPR.
%
{Note that, the membership neighborhood stage in our attack requires only access to the non-poisoned shadow models, allowing it to run in parallel as soon as the first iteration (Step 4, Algorithm \ref{alg:adp_psn_set}) of the adaptive poisoning stage concludes, incurring no additional time overhead. Our querying phase takes approximately a second to query 500 challenge points and their respective neighborhoods (each of size 64).
}

%% file: apndx_analysis.tex
\section{Analysis of Label-Only MI Under Poisoning} \label{apdx:MIAnalysis}

Let $\dist$ denote the distribution from which $n$ samples $z_1,\ldots,z_n$ are sampled in an iid manner. For classification based tasks, a sample is defined as  $z_i = (x_i,y_i)$ where $x_i$ denotes the input vector and $y_i$ denotes the class label. We assume binary membership inference variables $m_1, \ldots, m_n$ that are drawn independently with probability $\Pr(m_i = 1) = \lambda$, where samples with $m_i = 1$ are a part of the training set. 
We model the training algorithm as a random process such that the posterior distribution of the model parameters given the training data $\theta| z_1,\ldots,z_n,m_1, \ldots, m_n$ satisfies
\begin{align} \label{eqn:postdist}
    \Pr(\theta| z_1,\ldots,z_n,m_1,\ldots,m_n) \propto e^{-\frac{1}{\tau} \sum_{i=1}^n m_i L(\theta, z_i)}
\end{align}
where $\tau$ and $L$ denote the temperature parameter and the loss functions respectively. Parameter $\tau = 1$ corresponds to the case of the Bayesian posterior, $\tau \rightarrow 0$ the case of MAP (Maximum A
Posteriori) inference and a small $\tau$ denotes the case of averaged
SGD.  This assumption on the posterior distribution of the model parameters have also been made in prior membership inference works such as  \citet{sablayrolles2019whitebox} and \citet{ye2022enhanced}. The prior on $\theta$ is assumed to be uniform. 

Without loss of generality, let us analyze the case of $z_1 = (x_1, y_1)$, for a multi-class classification task. Let $C$ denote the total number of classes in the classification task.  We introduce poisoning by creating a poisoned dataset $\psndata$ which contains $k$ poisoned replicas of $z^p_1 = (x_1, y^p_1)$, where $y^p_1 \neq y_1$. Note that, all poisoned replicas have the same poisoned label $y^p_1$ that is distinct from the true label. The posterior in \Cref{eqn:postdist} can then be re-written as:
\begin{align} \label{eqn:psndist}
    \Pr(\theta_p| z_1,\ldots,z_n,m_1,\ldots,m_n,\psndata) \propto e^{-\frac{1}{\tau} \sum_{i=1}^n m_i L(\theta_p, z_i) - \frac{k}{\tau}L(\theta_p, z_1^p)} 
\end{align}

where term $\frac{k}{\tau} L(\theta_p, z_1^p)$ denotes the sum over all the loss terms introduced by $k$ poisoned replicas. 
Furthermore, we gather information about other samples and their memberships  in set $T = \{z_2,\ldots,z_n,m_2,\ldots,m_n\}$. We assume the loss function $L$ to be a 0-1 loss function, so that we can perform a concrete analysis of the loss in \Cref{eqn:psndist}. 

\paragraph{Assumptions.} 
{We now explicitly list the set of assumptions that will be utilized to design and analyze our optimal attack.
\begin{description}
    \item[-] The posterior distribution of the model parameters given the poisoned dataset satisfies \Cref{eqn:psndist},  where the loss function $L$ is assumed to be a 0-1 loss function.
    \item[-] The prior on the model parameters $\theta$ follows a uniform distribution.  
    \item[-] The model parameter selection for classifying challenge point $z_1$ is only dependent on $z_1$,  its membership $m_1$ and the poisoned dataset $D_p$.  We make this assumption based on the findings from \citet{TruthSerum}, where empirical observations revealed that multiple \emph{poisoned models} aimed at inferring the membership of $z_1$ had very similar logit (scaled confidence) scores for challenge point $z_1$. This observation implied that a poisoned model’s prediction on $z_1$ was largely influenced only by the presence/absence of the original point $z_1$ and the poisoned dataset $D_p$.  
\end{description}
}
\paragraph{Poisoning impact on challenge point classification.} Recall that we are interested in analyzing how the addition of $k$ poisoned replicas influences the correct classification of the challenge point $z_1$ as label $y_1$, considering whether $z_1$ is a member or a non-member. Towards this, we define the event $\theta_p(z_1) = y_1$  as selecting a parameter $\theta_p$ that correctly classifies $z_1$ as $y_1$. More formally, we can write the probability of correct classification as follows:
% We  are interested in  determining the probability of selecting a poisoned parameter $\theta_p$ that correctly classifies $z_1$ as $y_1$. We use notation $\theta_p(z_1) = y_1$  to denote the aforementioned event.

\begin{theorem}
    Given the sample $z_1$,  binary membership variable $m_1$, poisoned dataset $D_p$ and the remaining training set $T$.
    
    \begin{equation} \label{eqn:mem_final}
    \Pr(\theta_p(z_1) = y_1 | z_1,m_1, T,\psndata) = \frac{e^{-k/\tau}}{e^{-m_1/\tau} + e^{-k/\tau} + (C-2) e^{-(k+m_1)/\tau}}
    \end{equation}
\end{theorem}

\begin{proof}
Let us first consider the OUT case when $z_1$ is not a part of the training set, i.e. $m_1 = 0$.  Formally we can write the probability of selecting a parameter that results in classification of sample $z_1$ as label $y_1$ as follows:  
\begin{equation*}
\Pr(\theta_p(z_1) =y_1 | z_1,m_1= 0,T,\psndata)
\end{equation*}
Based on our assumption that under the presence of poisoning, the model parameter selection for  classification of $z_1$ depends only on $z_1$, $m_1$ and $\psndata$, we can write
\begin{equation} \label{eqn:mem}
\Pr(\theta_p(z_1) = y_1 | z_1,m_1= 0,T,\psndata) = \Pr(\theta_p(z_1) = y_1 | z_1,m_1= 0,\psndata)
\end{equation}

Subsequently, we can use \Cref{eqn:psndist} to reformulate \Cref{eqn:mem} as follows:
\begin{align}\label{eqn:outdist} 
\Pr(\theta_p(z_1) = y_1 | z_1,m_1= 0,\psndata) = \frac{e^{-k/\tau}}{e^0 + (C-1) e^{-k/\tau}} 
\end{align}
In this equation, the numerator $e^{-k/\tau}$ denotes the outcome where  a parameter $\theta_p$ is chosen resulting in classification of $z_1$ as $y_1$. Similarly, the terms $e^0$ and $(C-1)e^{-k/\tau}$ in the denominator represent the outcomes where a parameter is selected such that it classifies $z_1$ as $y_1^p$ and any other label except $y_1^p$, respectively.

Similar to \Cref{eqn:outdist}, we can formulate an equation for the IN case as follows:

\begin{align}\label{eqn:indist} 
\Pr(\theta_p(z_1) = y_1 | z_1,m_1= 1,\psndata) = \frac{e^{-k/\tau}}{e^{-1/\tau} + e^{-k/\tau} + (C-2) e^{-(k+1)/\tau}}
\end{align}

Similar to \Cref{eqn:outdist}, the numerator $e^{-k/\tau}$ denotes the outcome that classifies sample $z_1$ as $y_1$.The terms $e^{-1/\tau}$ and $e^{-k/\tau}$ in the denominator represent the outcomes when sample $z_1$ is classified as $y_1^p$ and $y_1$ respectively. Term $(C-2) e^{-(k+1)/\tau}$ denotes the sum over all outcomes where sample $z_1$ is classified as any  label except $y_1$ and $ y_1^p$. 
Now, by combining \Cref{eqn:outdist} and \Cref{eqn:indist}, we arrive at a unified expression represented by \Cref{eqn:mem_final}.
\end{proof}

% The goal of current state-of-the-art attacks \citep{LiRA,TruthSerum,wen2023canary,bertran2023scalable} is to construct an attack that can  maximize the TPR at a fixed FPR. 
\paragraph{Optimal Attack}
{ Our goal now is to formulate an \emph{optimal} attack in the label-only setting that maximizes the TPR value at fixed FPR of x\%, when $k$ poisoned replicas of $z_1^p$ are introduced into the training set, adhering to the list of assumptions defined earlier.
}

We define two events: 

\begin{description}
    \item[-] If $\theta(z_1) = y_1$, we say $z_1$ is "IN" the training set with probability $p_0$, else we say "OUT" with probability $1-p_0$.
    \item[-]  If $\theta(z_1) \neq y_1$, we say $z_1$ is "IN" the training set with probability $p_1$, else we say "OUT" with probability $1-p_1$.
\end{description}

We can then compute the maximum TPR as follows:
 
\begin{theorem}\label{thm:tpr}
Given sample $z_1$ and a training dataset that includes $k$ poisoned replicas of $z_1$. The maximum TPR at $x\%$ FPR is given as

\begin{align*}
   x' \times \frac{(C-1) + e^{k/\tau}}{e^{(k-1)/\tau} + (C-2) e^{-1/\tau} + 1} - p \times \frac{e^{k/\tau} - e^{(k-1)/\tau} + (C-2)(1-e^{-1/\tau})}{e^{(k-1)/\tau} + (C-2) e^{-1/\tau} + 1}
\end{align*}

where probability $p = max\left(0, \frac{x'(C-1) + x'e^{k/\tau} - 1}{(C-2) + e^{k/\tau}} \right)$ and $x' = x/100$.

\end{theorem}

\begin{proof}
Let $x' = x/100$. We can start by writing the equation for FPR as: 

\begin{align*} 
\Pr(\text{ "IN" }| ~z_1,m_1 = 0, T, \psndata) = x'
\end{align*}

We can expand the left-hand side of the equation as follows:
\begin{align*} 
\Pr(\text{ "IN" }| ~z_1,m_1 = 0, T, \psndata) = \Pr(\text{ "IN" }| ~\theta_p(z_1) = y_1) \times \Pr(\theta_p(z_1) = y_1 | z_1,m_1= 0,T, \psndata) 
\\+ \Pr(\text{ "IN" }| ~\theta_p(z_1) \neq y_1) \times \Pr(\theta_p(z_1) \neq y_1 | z_1,m_1= 0,T,\psndata)
\\ = p_0 \times \frac{1}{e^{k/\tau} + (C-1)}  + p_1 \times \frac{e^{k/\tau} + (C-2)}{e^{k/\tau} + (C-1)}
\end{align*}

At a fixed FPR $x'$, the above equation can then be re-written as:

\begin{equation} \label{eqn:p0-p1-relation}
p_0 = x' \times ( (C-1) + e^{k/\tau} ) - p_1 \times ((C-2) + e^{k/\tau})
\end{equation}

We also know that the following inequalities hold $0 \leq p_0,~ p_1 \leq 1$.  By substituting $p_0$ as a function of $p_1$ from \Cref{eqn:p0-p1-relation}, we get:
\begin{equation*}
max\left(0, \frac{x'(C-1) + x'e^{k/\tau} - 1}{(C-2) + e^{k/\tau}} \right) \leq p_1 \leq min\left(1, \frac{x'(C-1) + x'e^{k/\tau}}{(C-2) + e^{k/\tau}} \right)
\end{equation*}

% \begin{align*}
% \Pr(\text{ "IN" }| ~\theta_p(z_1) = y_1) = \frac{0.01 \times x}{\Pr(\theta_p(z_1) = y_1 | z_1,m_1= 0,\psndata)} = \frac{0.01 \times x}{\alpha_{m_1=0}}
% \end{align*}

Similar to the FPR equation, we formulate the TPR as follows:
\begin{align*} 
\Pr(\text{ "IN" }| ~z_1,m_1 = 1, T, \psndata) = \Pr(\text{ "IN" }| ~\theta_p(z_1) = y_1) \times \Pr(\theta_p(z_1) = y_1 | z_1,m_1= 1,T, \psndata) 
\\+ \Pr(\text{ "IN" }| ~\theta_p(z_1) \neq y_1) \times \Pr(\theta_p(z_1) \neq y_1 | z_1,m_1= 1,T,\psndata)
\\ = p_0 \times  \frac{1}{e^{(k-1)/\tau} + (C-2) e^{-1/\tau} + 1} + p_1 \times \frac{e^{(k-1)/\tau} + (C-2) e^{-1/\tau}}{e^{(k-1)/\tau} + (C-2) e^{-1/\tau} + 1}
\end{align*}

We substitute \Cref{eqn:p0-p1-relation} into the above equation and get:

\begin{align} \label{eqn:tpr_intrmid}
   = x' \times \frac{(C-1) + e^{k/\tau}}{e^{(k-1)/\tau} + (C-2) e^{-1/\tau} + 1} - p_1 \times \frac{e^{k/\tau} - e^{(k-1)/\tau} + (C-2)(1-e^{-1/\tau})}{e^{(k-1)/\tau} + (C-2) e^{-1/\tau} + 1}
\end{align}

The goal is to maximize the above TPR equation when FPR $=  x'$. We can then write \Cref{eqn:tpr_intrmid} as a constrained optimization problem as follows: 

\begin{equation} \label{eqn:tpr_opt}
\begin{aligned} 
\max_{p_1} \quad & x' \times \frac{(C-1) + e^{k/\tau}}{e^{(k-1)/\tau} + (C-2) e^{-1/\tau} + 1} - p_1 \times \frac{e^{k/\tau} - e^{(k-1)/\tau} + (C-2)(1-e^{-1/\tau})}{e^{(k-1)/\tau} + (C-2) e^{-1/\tau} + 1}\\
\\
\textrm{s.t.} \quad & max\left(0, \frac{x'(C-1) + x'e^{k/\tau} - 1}{(C-2) + e^{k/\tau}} \right) \leq p_1 \leq min\left(1, \frac{x'(C-1) + x'e^{k/\tau}}{(C-2) + e^{k/\tau}} \right)\\
\end{aligned}
\end{equation}

In \Cref{eqn:tpr_opt}, we observe that the first term is a constant and the coefficient of $p_1$ is positive. Consequently, we must set $p_1$ to its minimum possible value in order to maximize the TPR value. Thus the TPR equation can be re-written as: 

\begin{align} \label{eqn:tpr_eqn}
   = x' \times \frac{(C-1) + e^{k/\tau}}{e^{(k-1)/\tau} + (C-2) e^{-1/\tau} + 1} - p \times \frac{e^{k/\tau} - e^{(k-1)/\tau} + (C-2)(1-e^{-1/\tau})}{e^{(k-1)/\tau} + (C-2) e^{-1/\tau} + 1}
\end{align}

where probability $p = max\left(0, \frac{x'(C-1) + x'e^{k/\tau} - 1}{(C-2) + e^{k/\tau}} \right)$.

\end{proof}

As previously shown in Figure \ref{fig:Theoretical_TPR} (Section \ref{sec:Analysis}), we plot the TPR as a function of the number of poisoned replicas for our setting using Theorem \ref{thm:tpr}. We set the temperature parameter to a small value $\tau  = 0.5$, the number of classes $C = 10$.  We fix the FPR to $5\%$ and plot our theoretical attack. In order to validate the similarity in behavior between our theoretical model and practical scenario, we run the static version of our label-only attack where we add $k$ poisoned replicas for a challenge point in CIFAR-10 dataset. We observe that the TPR improves with increase with introduction of poisoning and then decreases as the number of poisoned replicas get higher. Note that, the assumptions made in our theoretical analysis do not hold in the absence of poisoning ($k = 0$). Hence, we see a discrepancy between the practical and theoretical attack at $k=0$.

\section{Privacy Game} \label{apdx:threatmodel}

We consider a privacy game, where the attacker aims to guess if a challenge point $(x,y) \sim \dist$ is present in the challenger's training dataset $\traindata$. The game between the challenger $\challenger$ and attacker $\Adv$ proceeds as follows: 
% Both have access to the underlying data distribution $\dist$, and know the set $\mathcal{U}$ and training algorithm $\mathcal{T}$. 
\begin{itemize} [noitemsep]
    \item [1: ] $\challenger$ samples training data $\traindata$ from the underlying distribution $\dist$.
    \item[2: ] $\challenger$ randomly  selects $b \in \{0, 1\}$. If $b = 0$, $\challenger$ samples a point $(x, y) \sim \dist$ uniformly at random, such that $(x,y) \notin \traindata$ . Else, samples $(x, y)$ from $\traindata$ uniformly at random.
    \item [3: ]  $\challenger$ sends the challenge point $(x, y)$ to $\Adv$.  
    \item [4: ]  $\Adv$ constructs a poisoned dataset $\psndata$ and sends it to $\challenger$.
    
    \item [5: ]  $\challenger$ trains a target model $\theta_t$  on the poisoned dataset $\traindata \cup \psndata$.

    \item [6: ] $\challenger$ gives $\Adv$ label-only access to the target model $\theta_t$.
    \item [7: ]  $\Adv$ queries the target model $\theta_t$, guesses a bit $\hat{b}$ and wins if $\hat{b} = b$. 
    
\end{itemize}

The challenger $\challenger$  samples training data $\traindata \sim \dist$ from an underlying data distribution $\dist$. The attacker $\Adv$  has the capability to inject additional poisoned data $\psndata$ into the training data $\traindata$.
The objective of the attacker is to enhance its ability to infer if a specific point $(x,y)$ is present in the training data by interacting with a model trained by challenger  $\challenger$ on data $\traindata \cup \psndata$.  The attacker can only inject $\psndata$ once before the training process begins, and after training, it can only interact with the final trained model to obtain predicted labels. Note that, both the challenger and the attacker have access to the underlying data distribution $\dist$, and know the challenge point $(x,y)$ and training algorithm $\mathcal{T}$, similar to prior works \citep{LiRA,TruthSerum,MIPoison22,wen2023canary}.